\documentclass{article}

\usepackage{amsmath,amsfonts,bm}

\def\1{\bm{1}}

\def\rt{{\textnormal{t}}}

\def\rw{{\textnormal{w}}}
\def\rx{{\textnormal{x}}}
\def\ry{{\textnormal{y}}}
\def\rz{{\textnormal{z}}}

\def\rvc{{\mathbf{c}}}

\def\rve{{\mathbf{e}}}

\def\rvu{{\mathbf{i}}}

\def\rvt{{\mathbf{t}}}
\def\rvu{{\mathbf{u}}}
\def\rvv{{\mathbf{v}}}
\def\rvw{{\mathbf{w}}}
\def\rvx{{\mathbf{x}}}
\def\rvy{{\mathbf{y}}}
\def\rvz{{\mathbf{z}}}

\def\vtheta{{\bm{\theta}}}
\def\va{{\bm{a}}}
\def\vb{{\bm{b}}}
\def\vc{{\bm{c}}}
\def\vd{{\bm{d}}}

\def\vf{{\bm{f}}}
\def\vg{{\bm{g}}}
\def\vh{{\bm{h}}}

\def\vk{{\bm{k}}}

\def\vr{{\bm{r}}}
\def\vs{{\bm{s}}}
\def\vt{{\bm{t}}}
\def\vu{{\bm{u}}}
\def\vv{{\bm{v}}}

\def\vx{{\bm{x}}}
\def\vy{{\bm{y}}}
\def\vz{{\bm{z}}}

\def\mA{{\bm{A}}}

\def\mC{{\bm{C}}}

\def\mG{{\bm{G}}}

\def\mI{{\bm{I}}}
\def\mJ{{\bm{J}}}

\def\mL{{\bm{L}}}

\def\mO{{\bm{O}}}

\def\mR{{\bm{R}}}

\def\mX{{\bm{X}}}

\DeclareMathAlphabet{\mathsfit}{\encodingdefault}{\sfdefault}{m}{sl}
\SetMathAlphabet{\mathsfit}{bold}{\encodingdefault}{\sfdefault}{bx}{n}

\newcommand{\E}{\mathbb{E}}

\newcommand{\R}{\mathbb{R}}

\newcommand{\KL}{\mathbb{D}}
\newcommand{\Var}{\mathrm{Var}}

\usepackage{url}

\usepackage{amsmath,amssymb,amsfonts,amsthm,bm}
\newtheorem{theorem}{Theorem}
\newtheorem{proposition}{Proposition}
\newtheorem{corollary}{Corollary}
\theoremstyle{definition}
\newtheorem{definition}{Definition}
\usepackage{mathtools}
\usepackage{graphicx}
\usepackage{paralist}
\usepackage[linesnumbered,ruled,noend]{algorithm2e}
\usepackage{xcolor}

\SetCommentSty{mycommfont}
\usepackage{subcaption}
\usepackage{tabularx}
\usepackage{booktabs}
\usepackage{caption}
\usepackage{wrapfig}
\usepackage{centernot}
\usepackage{enumitem}   

\newcommand{\independent}{\raisebox{0.05em}{\rotatebox[origin=c]{90}{$\models$}}}
\newcommand{\notindependent}{\centernot{\independent}}
\newcommand{\deq}{\overset{d}{=}}

\newcommand{\xt}{|_{\x,\rt}}
\DeclareMathOperator{\diag}{diag}

\DeclareMathOperator{\bern}{Bern}
\DeclareMathOperator{\logi}{Logistic}
\newcommand{\blambda}{\bm\lambda}
\newcommand{\btheta}{\bm\theta}
\newcommand{\bphi}{\bm\phi}
\newcommand{\bgamma}{\bm\gamma}
\newcommand{\bbeta}{\bm\beta}
\newcommand{\beps}{\bm\epsilon}
\newcommand{\inv}{^{-1}}
\newcommand{\st}{^*}
\newcommand{\PS}{\mathbb{P}}
\newcommand{\BS}{\mathbb{B}}
\newcommand{\F}{\mathbb{F}}
\newcommand{\x}{\rvx}
\newcommand{\y}{\rvy}
\newcommand{\z}{\rvz}
\newcommand{\rconf}{\rvu}
\newcommand{\conf}{\vu}

\newcommand{\vaeobsparam}{p_{\vtheta}(\y|\x,\rt)}

\newcommand{\vaepostparam}{p_{\vf,\bm\lambda}(\z|\x,\y,\rt)}

\newcommand{\trueobs}{p^*(\y|\x,\rt)}

\setlist{nosep} 
\AtBeginDocument{%
\setlength{\belowdisplayskip}{2pt plus 1.0pt minus 2.0pt} \setlength{\belowdisplayshortskip}{2pt plus 1.0pt minus 2.0pt}
\setlength{\abovedisplayskip}{3pt plus 1.0pt minus 2.0pt} \setlength{\abovedisplayshortskip}{3pt plus 1.0pt minus 2.0pt}
}
\makeatletter
\def\thm@space@setup{%
  \thm@preskip=3.0pt plus 1.0pt minus 2.0pt
  \thm@postskip=2.0pt plus 1.0pt minus 2.0pt
}
\makeatother

\theoremstyle{plain}
\newtheorem{lemma}{Lemma}

\usepackage{microtype}
\usepackage{graphicx}
\usepackage{booktabs} 

\usepackage{hyperref}

\usepackage[accepted]{icml2021}

\icmltitlerunning{Intact-VAE: Estimating Treatment Effects under Unobserved Confounding}

\begin{document}

\twocolumn[
\icmltitle{Intact-VAE: Estimating Treatment Effects under Unobserved Confounding}



\icmlsetsymbol{equal}{*}

\begin{icmlauthorlist}
\icmlauthor{Pengzhou (Abel) Wu}{sokendai,ism}
\icmlauthor{Kenji Fukumizu}{sokendai,ism}

\end{icmlauthorlist}

\icmlaffiliation{sokendai}{Department of Statistical Science, The Graduate University for Advanced Studies, Tachikawa, Tokyo}
\icmlaffiliation{ism}{The Institute of Statistical Mathematics, Tachikawa, Tokyo}

\icmlcorrespondingauthor{Pengzhou (Abel) Wu}{wu.pengzhou@ism.ac.jp}

\icmlkeywords{VAE, variational autoencoder, representation Learning, treatment effects, causal inference, Unobserved Confounding, identifiability, CATE, ATE}

\vskip 0.3in
]



\printAffiliationsAndNotice{} 

\begin{abstract}
As an important problem of causal inference, we discuss the identification and estimation of treatment effects under unobserved confounding. Representing the confounder as a latent variable, we propose Intact-VAE, a new variant of variational autoencoder (VAE), motivated by the prognostic score that is sufficient for identifying treatment effects. We theoretically show that, under certain settings, treatment effects are identified by our model, and further, based on the identifiability of our model  (i.e., determinacy of representation), our VAE is a consistent estimator with representation balanced for treatment groups. Experiments on (semi-)synthetic datasets show state-of-the-art performance under diverse settings.
\end{abstract}

\section{Introduction}


Causal inference \citep{imbens2015causal, pearl2009causality}, i.e, estimating causal effects of interventions, is a fundamental problem across many domains. In this work, we focus on the estimation of treatment effects, such as effects of public policies or a new drug, based on a set of observations consisting of binary labels for treatment / control (non-treated), outcome, and other covariates. The fundamental difficulty of causal inference is that we never observe \textit{counterfactual} outcomes, which would have been if we had made another decision (treatment or control). While the ideal protocol for causal inference is randomized controlled trials (RCTs), they often have ethical and practical issues, or suffer from prohibitive costs. Thus, causal inference from observational data is indispensable. 
It introduces other challenges, however. The most crucial one is \textit{confounding}: there may be variables (called \textit{confounders}) that causally affect both the treatment and the outcome, and spurious correlation follows. 

A large majority of works in causal inference rely on the \textit{unconfoundedness}, which means that appropriate covariates are collected so that the confounding can be controlled by conditioning on those variables. That is, there is in essence \textit{no} unobserved confounders. This is still challenging, due to systematic imbalance (difference) of the distributions of the covariates between the treatment and control groups. One classical way of dealing with this difference is re-weighting \citep{horvitz1952generalization}. There are semi-parametric methods \citep[e.g. TMLE,][]{van2011targeted}, which have better finite sample performance, and also non-parametric,  tree-based methods \citep[e.g.~Causal Forests (CF),][]{wager2018estimation}. Notably, there is a recent rise of interest in learning \textit{balanced} representation of covariates, which is independent of treatment groups, starting from \citet{johansson2016learning}. 

There are a few lines of works that challenge the difficult but important problem of causal inference under \textit{unobserved confounding}, where the methods in the previous paragraph fundamentally does not work. Without covariates we can adjust for, many of them assume special structures among the variables, such as instrumental variables (IVs) \citep{angrist1996identification}, proxy variables \citep{miao2018identifying}, network structure \citep{ogburn2018challenges}, and multiple causes \citep{wang2019blessings}. Among them, instrumental variables and proxy (or surrogate) variables are most commonly exploited. \textit{Instrumental variables} are not affected by unobserved confounders, influencing the outcome only through the treatment. On the other hand, \textit{proxy variables} are causally connected to unobserved confounders, but are not confounding the treatment and outcome by themselves. Other methods use restrictive parametric models \citep{allman2009identifiability}, or only give interval estimation \citep{manski2009identification, kallus2019interval}.

In this work, we address the problem of estimating treatment effects under unobserved confounding. The naive regression with observed variables introduces bias, if the decision of treatment and the outcome are confounded, as explained in Sec.~\ref{sec:setup}. 
To model the problem, we regard the confounder as a latent variable in representation learning, and propose a VAE-based method that identifies the treatment effects. We in particular discuss the \textit{individual-level} treatment effect, which measures the treatment effect conditioned on the covariate, for example, on a patient's personal data. 


Our method is based firmly on established results in causal inference. Naturally combined with the important concepts of sufficient scores \cite{hansen2008prognostic,rosenbaum1983central}, our theory show identification of treatment effects \textit{without} needs to recover hidden confounder or true scores. 

Our method also exploits the recent advance of \textit{identifiable} VAE \citep[iVAE]{khemakhem2020variational}. 
The hallmark of deep neural networks (NNs) is that they can learn representations of data. 
A principled approach to interpretable representations is identifiability, that is, when optimizing our learning objective w.r.t.~the representation function, only a unique optimum, which represents the true latent structure, will be returned. Our method provides the stronger identifiability that gives \textit{balanced} representation.

The main \textbf{contributions} of this paper are as follows: 
\vspace{-5pt}
\begin{enumerate}[topsep=0pt, partopsep=0pt, itemsep=0pt, parsep=0pt, leftmargin=13pt]
\item[1)] A new identifiable VAE as a balanced estimator for treatment effects under unobserved confounding;
\item[2)] 
Theory, with newly introduced B*-score, of identifiability, identification, and estimation of treatment effects;
\item[3)] Experimental comparisons with state-of-the-art methods.  
\end{enumerate}


\subsection{Related Work}


\textbf{Identifiability of representation learning.} With recent advances in nonlinear ICA, identifiability of representations is proved under a number of settings, e.g., auxiliary task for representation learning \citep{hyvarinen2016unsupervised, hyvarinen2019nonlinear} and VAE \citep{khemakhem2020variational}. Recently, \citet{roeder2020linear} extends the the result to include a wide class of state-of-the-art deep discriminative models. The results are exploited in bivariate causal discovery \citep{pmlr-v108-wu20b} and structure learning \citep{yang2020causalvae}. To the best of our knowledge, this work is the \textit{first} to explore this 
identifiability in inference on treatment effects.  

\textbf{Representation learning for causal inference.} Recently, researchers start to design representation learning methods for causal inference, but mostly limited to \textit{unconfounded} settings. Some methods focus on learning a balanced representation of covariates, e.g., BLR/BNN \citep{johansson2016learning}, and TARnet/CFR \citep{shalit2017estimating}. Adding to this, \cite{yao2018representation} also exploits the local similarity of between data points. \cite{shi2019adapting} uses similar architecture to TARnet, considering the importance of treatment probability. There are also methods using GAN \citep[GANITE]{yoon2018ganite} and Gaussian process \citep{alaa2017bayesian}. Our method shares the idea of balanced representation learning, and further extends to the harder problem of \textit{unobserved confounding}.

\textbf{Causal inference with auxiliary structures.}
Both of our method and CEVAE \citep{louizos2017causal} use VAE as a learning method. Except this apparent similarity, our method is quite different to CEVAE in motivation, applicability, architecture, and, particularly, theoretical development. Note that CEVAE relies on the strong assumption that VAE can recover the true confounder distribution. More detailed comparisons are given in Appendix. Under linear model, \citet{kallus2018causal} use matrix factorization to infer the confounders from proxy variables, and give consistent ATE estimator together with its error bound. \citet{miao2018identifying} established conditions for identification using more general proxies, but without practical estimation method. Additionally, two active lines of works in machine learning exist in their own right, exploiting IV \citep{hartford2017deep} and network structure \citep{veitch2019using}. Different to the above methods, our method is based on more general concepts of sufficient scores for causal inference, \textit{without} assuming specific auxiliary structures. Also, our method gives consistent estimator for \textit{nonlinear} outcome models, given consistent learning of VAE. 


\begingroup
\setlength{\columnsep}{8pt}
\setlength{\intextsep}{5pt}
\begin{wrapfigure}{r}{0.25\columnwidth}
  \vspace{-.2in}
  \begin{center}
    \includegraphics[width=0.25\columnwidth]{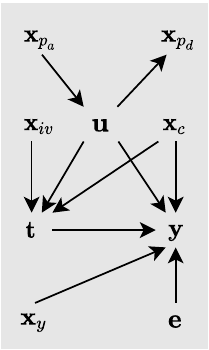}
  \end{center}
  \vspace{-.2in}

  \caption{
  \footnotesize{A typical causal graph of our setting.}
  } 
  
\vspace{-.1in}
\label{fig:setting}
\end{wrapfigure}

\section{Setup: Treatment Effects and Identification}
\label{sec:setup}

Following \citet{imbens2015causal}, we begin by introducing  \textit{potential outcomes} (or \textit{counterfactual outcomes}) $\rvy(t), t \in \{0,1\}$. $\rvy(t)$ is the outcome we \textit{would} observe, if we applied treatment value $\rt=t$. Note that, for a unit under research, we can observe only one of $\y(0)$ or $\y(1)$, corresponding to which factual treatment we have applied, do not observe the other. This is the \textit{fundamental problem of causal inference}.  We often observe relevant covariate $\rvx$, which is associated with individuals, and the observation $(\rvx,\rvy,\rt)$ is a random variable with underlying probability distribution. 

The expected potential outcome is denoted by $\mu_t(\vx) = \E(\rvy(t)|\rvx=\vx)$, conditioned on covariate $\rvx=\vx$.
The estimands in this work are the causal effects,  Conditional Average Treatment Effect (CATE) $\tau$ and Average Treatment Effect (ATE) $\nu$, defined respectively by 
\vspace{-.05in}
\begin{equation}
\label{eq:ce}
    \tau(\vx) = \mu_1(\vx) - \mu_0(\vx),\quad \nu = \E(\tau(\vx)).
\end{equation}
CATE can be understood as an \textit{individual-level} treatment effect, if conditioned on
highly discriminative covariates.

\textit{Identifiability of treatment effects} means that, the true, \textit{observational} distribution uniquely determines the ATE or, better, CATE. Further, \textit{identification of treatment effects} means that, we can derive an equation of treatment effects, given the observational distribution. 

As the standard ones \citep{rubin2005causal}\citep[Ch.~3]{hernanCausalInferenceWhat2020}, we make the following assumption for the data generating probability throughout this paper: \\
{\bf (A)} 
there exists a random variable $\rconf \in \R^n$ such that (i) together with $\rvx$, it satisfies \textit{exchangeability}
($\rvy(t) \independent \rt| \rconf, \rvx$\footnote{It is satisfied for \textit{both} $t$. In this paper, when $t$ appears in a statement without quantification, it always means ``for both $t$''.}) (ii) \textit{positivity} ($p(\rt|\rconf, \rvx) > 0$), and (iii) \textit{consistency} of counterfactuals ($\rvy = \rvy(t)$ if $\rt=t$). 
We say that \textit{weak ignorability} holds when we have both exchangeability and positivity.

Exchangeability means there is no correlation between factual assignment of treatment and counterfactual outcomes, given $\rconf,\rvx$, just as it is the case in an RCT. 
Thus, it can be understood as \textit{unconfoundedness} given $\rconf,\rvx$, and $\rconf$ can be seen as the unobserved confounder(s). Positivity says the supports of $p(\rt=t|\rconf,\rvx)$ should always be \textit{overlapped}, and this ensures there are no impossible events in the conditions after adding $\rt=t$.
Finally, consistent counterfactuals are well defined. In causal inference, we see $\rvy(t)$ as the underlying hidden variables that give \textit{factual} (observational) $\y$ if $\rt=t$ is assigned. If, say, we assigned treatment $\rt=1$, we observe $\vy=\vy(1)$, a realization of $\y(1)$. We understand that, there exists also $\y(0)$ corresponding to the random outcome we would observe, if we had applied $\rt=0$, the counterfactual assignment. 

In Figure \ref{fig:setting}, $\rvx_{c}$,$\rvx_{iv}$,$\rvx_{p_a}$,$\rvx_{p_d}$,$\rvx_{y}$ are covariates that are: (observed) confounder, IV, antecedent proxy (that is antecedent of $\rvz$), descendant proxy, and antecedent of $\rvy$, respectively. The covariate(s) $\rvx$ may \textit{not} have subsets in any categories in the graph. The assumptions may hold otherwise, e.g., $\rvx$ is a child of $\rt$. And $\rve$ is unobserved noise on $\rvy$ (used in Sec.\ref{sec:treatment}).

CATE can be given by \eqref{eq:id} (\textbf{(A)} used in the second equality).
\begin{equation}
\label{eq:id}
\begin{split}
    \mu_t(\vx) &= \E(\E(\rvy(t)|\rconf,\vx)) = \E(\E(\rvy|\rconf,\vx,\rt=t)) \\ &=\textstyle \int (\int p(y|\conf,\vx,t)ydy)p(\conf|\vx)d\conf.
\end{split}
\end{equation}

However, the variable $\rconf$ above is an \textit{unobserved confounder}.
Due to it, assumption \textbf{(A)} does not ensure identifiability and \eqref{eq:id} is not an identification. 
Note that, under unobserved confounding, the naive regression $\E[\rvy| \rvx=\vx, \rvt=t]$ based on observable variables is not equal to $\mu_t(\vx)$. In fact, if an unknown factor correlates with $\rvt$ positively and tends to give higher value for $\rvy$, the naive regression $\E[\rvy| \rvx=\vx, \rvt=1]$ should be higher than $\E[\rvy(1)|\rvx=\vx]$.

\section{Intact-VAE}
In this section, we introduce B*-scores motivated by prognostic scores (Sec.~3.1), and our VAE model and architecture, based on the distribution $p(\rvy,\rvz|\rvx,\rt)$ (Sec.~\ref{sec:arch})
The developments give many hints to causal inference, and enable us to address identification and estimation of treatment effects in Sec.~\ref{sec:treatment}. 

\subsection{Motivation}

Our method is motivated by prognostic scores \citep{hansen2008prognostic}, adapted as \textit{P*-scores} in this paper, closely related to the important concept of balancing score \citep{rosenbaum1983central}. Both are sufficient scores (statistics) for identification, but P*-score is arguably more applicable, and, combined with our model, it motivates an identifiable VAE.

\begin{definition}[P*-scores]
A \textit{P0-score} (or \textit{P1-score}) of random variable $\rvv$ is a function $\PS(\rvv)$ such that $\rvy(0)\independent\rvv|\PS(\rvv)$ (or $\rvy(1)\independent\rvv|\PS(\rvv)$). Given a P0-score $\PS_0$ and a P1-score $\PS_1$, a \textit{Pt-score} is defined as $\PS_{\rt}$ (i.e. $\PS_t$ if $\rt=t$). A Pt-score is called a \textit{P-score} if $\PS\coloneqq\PS_0=\PS_1$.
\end{definition}

The sufficiency of P*-score inspires the following definition.

\begin{definition}[B*-scores]
\label{def:bscore}
Let $p(\x,\y,\rt,\rconf,\BS_0)$ be a true generating distribution.
$\BS_0$ is a B0-score (for $p(\y|\x,\rt)$) if $\E(\rvy(0)|\BS_0,\rvx)=\E(\rvy(0)|\BS_0,\rt=0)\coloneqq \F_0(\BS_0)$ for any $\BS_0=B_0,\x=\vx$. 
B1-score, Bt-score, and B-score are also defined, as in the similar way that P*-scores are defined relative to P0-score.
\end{definition}

The point is that, Bt-score is yet another weaker, but still sufficient score. Theorem \ref{cate_by_bts} and its corollary, which is also the key of Pt-score, is clearly valid for Bt-score. 





Our goal is to build a VAE that can learn from observational data to obtain a Bt-score or, more ideally B-score, by using the latent variable $\z$ of the VAE. This latent variable can be seen as a \textit{causal representation}, which can be used to identify and estimate treatment effects by \eqref{eq:cate_by_bts} or \eqref{eq:cate_by_bs}. 
Recovering the true confounder $\rconf$ is not necessary. 

Bt-score relaxes the \textit{independence property of Pt-score} (Proposition \ref{prop:pscore} in Appendix),
\begin{equation}
\label{eq:ps_indp}
    \rvy(t)\independent\rt,\rconf,\rvx|\PS_t
\end{equation}
to the \textit{mean equality} in Definition \ref{def:bscore}, that is sufficient. Both are used in second equality of \eqref{eq:cate_by_bts}. 
\begin{theorem}[CATE by Bt/Pt-score]
\label{cate_by_bts}
If $\BS_{\rt}$ is a Bt/Pt-score and $p(\rt|\BS_{\rt})>0$ 
, then CATE can be given by
\begin{equation}
\label{eq:cate_by_bts}
    \begin{split}
        \mu_{t}(\vx) &= \E(\E(\rvy({t})|\BS_{t},\vx)) = \E(\E(\rvy|\BS_{t},\rt={t})) \\ &=\textstyle \int (\int p(y|B_{t},{t})ydy)p(B_{t}|\vx)dB_{t}
    \end{split}
\end{equation}
\end{theorem}
 Importantly, compared to \eqref{eq:id}, $p(y|B_{t},{t})$ in \eqref{eq:cate_by_bts} does \textit{not} depend on $\x$, easily understood as $\rvy\independent \rvx|\BS,\rt$ implied by \eqref{eq:ps_indp}. 

We have a corollary simply by $\BS_0=\BS_1$ for a B-score.
\begin{corollary}[CATE by B/P-score]
\label{cate_by_bs}
If $\BS$ is a B/P-score, then CATE can be given by
\begin{equation}
\label{eq:cate_by_bs}
    \mu_{t}(\vx)=\textstyle \int (\int p(y|B,{t})ydy)p(B|\vx)dB
\end{equation}
\end{corollary}
Also important here is the difference to \eqref{eq:cate_by_bts} is only\footnote{Note $p(y|B_{t},{t})=p(y|B,{t})$ also in \eqref{eq:cate_by_bts}, since $\rt=t$ is given.} in $p(B|\vx)$, which does \textit{not} depend on $\rt$ given $\x$. 

We should note that, both B0/B1-scores can depend on $\rt$, so, a B-score $\BS\coloneqq\BS_0=\BS_1$ can depend on $\rt$, despite the name might suggest. This is also true for P-score $\PS$ seeing as a random variable. We use the same symbol to denote a B*/P*-score and the random variable defined by it, when appropriate.

\subsection{Model and Architecture}
\label{sec:arch}
The generative model of our VAE is 
\begin{equation}
\label{model_indep}
    p(\rvy,\rvz|\rvx,\rt) = p(\rvy|\rvz,\rt)p(\rvz|\rvx,\rt).
\end{equation}
The correspondence between \eqref{eq:cate_by_bts} and \eqref{model_indep} lays the \textbf{first foundation} of our method. 
Note that $p(\BS_t|\x)$ in \eqref{eq:cate_by_bts} means $\BS_{\rt}$ (so, also $\z$) should depend on $\rt$ given $\x$.

We are steps away from our VAE architecture now. 
The major jump is noticing that \eqref{model_indep} has a similar factorization with the generative model of iVAE (see Appendix for details), that is $p(\rvy,\rvz|\x) = p(\rvy|\rvz)p(\rvz|\x)$. Note that the first factor does \textit{not} depend on $\x$, and this behavior is shared by our covariate $\rvx$ in \eqref{model_indep}.



Similarly to iVAE, $p(\rvy|\rvz,\rt)$ is our decoder, and $p(\rvz|\rvx,\rt)$ is our \textit{conditional} prior. Further, since we have the conditioning on treatment $\rt$ in \textit{both} factors of \eqref{model_indep}, our VAE architecture should be a combination of iVAE and conditional VAE (CVAE, see Appendix), with \textit{treatment} $\rt$ as the conditioning variable. The ELBO can be derived from
\begin{equation}
\label{elbo}
\begin{split}
    \log p(\rvy|\rvx,\rt) \geq \log p(\rvy|\rvx,\rt) - \KL(q(\rvz|\rvx,\rvy,\rt)\Vert p(\rvz|\rvx,\rvy,\rt)) \\ 
    = \E_{\vz \sim q}\log p(\rvy|\vz,\rt) - \KL(q(\rvz|\rvx,\rvy,\rt)\Vert p(\rvz|\rvx,\rt)) .
\end{split}
\end{equation}
Note again that, our decoder corresponds to outcome distribution $p(\y|\BS,\rt)$ in \eqref{eq:cate_by_bts} and our conditional prior to score distribution $p(\BS_t|\x)$.
Our encoder $q$, which conditions on all observables, is standard, and we will see its importance later. We name this architecture \textit{Intact-VAE} (\textit{I}de\textit{n}tifiable \textit{t}re\textit{a}tment-\textit{c}ondi\textit{t}ional VAE). 

\begingroup

\setlength{\columnsep}{8pt}
\setlength{\intextsep}{5pt}
\begin{wrapfigure}{r}{0.45\columnwidth}
  \vspace{-.3in}
  \begin{center}
    \includegraphics[width=0.45\columnwidth]{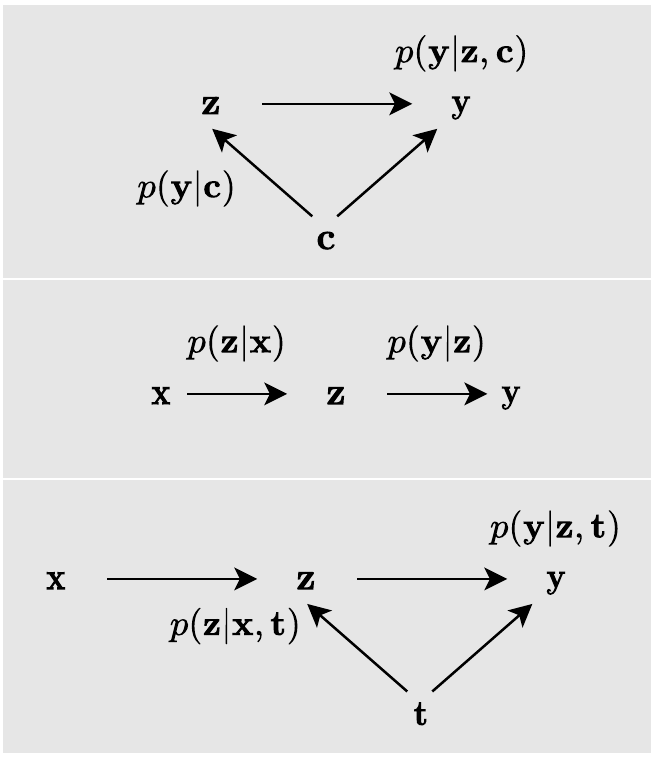}
  \end{center}
  \vspace{-.2in}
  \caption{Graphical models of the decoders. From top: CVAE, iVAE, and Intact-VAE. The encoders are similar: they take all observables and build approximate posteriors.} 
\vspace{-.1in}
\label{f:vae}
\end{wrapfigure}

Figure \ref{f:vae} depicts the relationship of CVAE, iVAE, and Intact-VAE. Do \textit{not} confuse the graphical model of our generative \textit{model} with a \textit{causal} graph. Particularly, do not confuse confounder $\rconf$ with latent variable $\z$ that corresponds to a Bt-score. For example, in Figure \ref{fig:setting}, the true, causal generating process we want to address,  confounder $\rconf$ should be the cause of $\rt$, and the causal arrow between $\rconf$ and $\rvx$ could also be reversed. 

Under reasonable assumptions in Sec.~\ref{sec:treatment}, our method is applicable to quite general causal settings in Figure \ref{fig:setting}, as we will show in Theorem \ref{id_tion}.
In particular, the \textit{observational} distribution $p(\rvy|\rvx,\rt)$ generated by our model can be the same as truTheorem

\endgroup




Readers may notice that, setting $p(\rvz|\rvx,\rt)=p(\rvz|\rvx)$ in \eqref{model_indep}
corresponds to \eqref{eq:cate_by_bs} (and also \eqref{eq:id} if we look only at $p(\rconf|\x)$). Indeed, later in Sec.~\ref{sec:estimation}, we will see that, given existence of a B/P-score, this could enable us to recover a B-score as representation, and give a \textit{balanced} estimator.

We detail the parameterization of Intact-VAE. For tractable inference and easy implementation, the decoder $p_{\vf,\vg}(\rvy|\rvz,\rt)$, conditional prior $p_{\vh,\vk}(\rvz|\rvx,\rt)$, and encoder $q_{\vr,\vs}(\rvz|\rvx,\rvy,\rt)$ are factorized Gaussians, i.e., a product of 1-dimensional Gaussian distributions, though our theory (Theorem \ref{idmodel} in Appendix) allows more general distributions. And this is not restrictive if the mean and variance are given by arbitrary nonlinear functions. 
\begin{equation} 
\label{model_param}
\textstyle
    \rvy|\rvz,\rt \sim \prod_{j}^d \mathcal{N}(y_j;f_j, g_j), 
    \rvz|\rvx,\rt \sim \prod_{i}^n \mathcal{N} ( z_i ; h_i , k_i) 
\end{equation}
\vspace{-.2in}
\begin{equation}
\label{eq:enc}
\textstyle
    \rvz|\rvx,\rvy,\rt \sim \prod_{i=1}^n \mathcal{N}(z_i ; r_i , s_i)
\end{equation}
$\vtheta=(\vf,\vg,\vh,\vk)$ and $\bm\phi=(\vr,\vs)$ are functional parameters given by NNs which take the respective conditional variables as inputs (e.g. $\vh\coloneqq(h_i(\rvx,\rt))^T$).

The \textit{identifiability} of our model, that is, \textit{parameters} $(\vf,\bm\lambda)$ can be identified (learned) up to affine transformation, is not needed until Sec.~\ref{sec:estimation} and thus saved to Theorem \ref{idmodel} in Appendix.
Throughout the paper, we will use subscript to denote the version of a quantity when $\rt=t$ is given, e.g., $\vh_t(\rvx)\coloneqq(h_i(\rvx,t))^T$

\section{Identification and Estimation}
\label{sec:treatment}

For identifiability of treatment effects, we need further assumptions on true generating process. 
In Sec.~\ref{sec:identifiability}, we give two examples of data generating process that are expressed by observables, and thus satisfy identifiability.
Then, we present the main theoretical results of this paper. First, our model can recover a Bt-score, and thus identify treatment effects (Sec.~\ref{sec:identification}), under some assumptions generalized from the examples. Second, our Intact-VAE is a consistent estimator for CATE, or even more ideally, balanced estimator (Sec.~\ref{sec:estimation}) with stronger model identifiability than iVAE.

\subsection{Examples of Identifiable Generating Process}
\label{sec:identifiability}





\textbf{Example 1.} 
Consider the generating process $\y=\vf^*(\PS^*( \rvx,\rconf,\rve_b),\rt)+\rve_a$ where $\vf^*,\PS^*$ are functions and $\rve=(\rve_a,\rve_b)$ is a noise such that $\rve_a \independent \rvx,\rconf,\rve_b|\PS^*$. Recall the definition, $\PS^*$ is a \textit{P-score} (thus also a B-score). However,  $\PS^*$ does not give identifiability since it is a function of  unobservables.
If, further, $\rve_a=\bm0$ and $\vf_t^*$ is \textit{injective}, then $\PS^*={\vf}^{*-1}_{\rt}(\rvy)$, we have identifiability. 

The next example uses a \textit{B-score} that is a function of $\x$ only. 

\textbf{Example 2.} Consider $\y=\vf^*(\BS^*(\rvx), \rconf,\rve_b),\rt)+\rve_a$.
Again by \eqref{eq:ps_indp}, $\BS^*$ is not a P-score, because $\y\notindependent\rconf|\BS\st$. 
However, $\BS^*$ can be a B-score in some cases, e.g., under linear $\vf\st$ (see Appendix for detail). We then have identifiability, because $\BS^*$ depends only on observable $\x$.



In Example 1, we prove that a Pt-score (but not P-score, recall Corollary \ref{cate_by_bs} for diffrence) is given by ${\bar{\vf}}^{-1}_{\rt}(\rvy)$ where $\bar{\vf}$ is \textit{any injective} functional parameter. The injectivity of $\bar{\vf},\vf\st_t$ is the key, it implies, up to an injective mapping, ${\bar{\vf}}^{-1}_{\rt}(\rvy)$ and $\PS\st$ have the same distribution.
To our comfort, corresponding to $\rve_a=\bm0$, set $\vg_t(\rvz) = \bm0$ in model, then $\rvz=\bar{\vf}^{-1}_{\rt}(\rvy)$ is given by our degenerate \textit{posterior} model $p_{\bar{\vf},\bm\lambda}(\z|\rvy,\rvx,\rt)=\delta(\rvz-\bar{\vf}^{-1}_{\rt}(\rvy))$ ($\delta$ function).

In Example 2, setting $\vk_t(\rvz) = \bm0$ in model, we prove that a Bt-score is given by $\bar{\vh}_{\rt}(\rvx)$, the degenerate conditional \textit{prior}. Again, 
the distributions of $\z$ (model) and $\BS\st$ (truth) is different only up to an injective mapping. Inspired by Example 1, we require again parameter $\vf_t$ \textit{and} the $\F_t$ corresponding to $\BS\st$ to be \textit{injective}. As to the \textit{noise}, to match our model and truth, we assume $\rve_a$ to be factorized Gaussian, and our model can learn it.

To sum up, in Example 1, we have P-score by our posterior model depending \textit{only} on $\y$, but we require zero noise on $\y$ to recover it. In Example 2, we recover B-score by our prior model depending \textit{only} on $\x$, while we can have noise. These are ample evidence for generalization. Particularly, our posterior depends on \textit{both} $\x,\y$, there should be more general cases where our method recovers a B-score depending on $\x,\y$. Indeed, this is what we have in Theorem \ref{id_tion}. 






\subsection{Identification}
\label{sec:identification}
We first establish identification, which means the best functional parameters can express the true treatment effects, before discussing learning in Sec.~\ref{sec:estimation}. The identification can be regarded as nonparametric, since we allow the functional parameters to realize arbitrary complex functions. When we realize the functions by NNs, by the universal approximator property for distributions \citep{lu2020universal}, they can be approximated with arbitrary accuracy. 

We summarize lessons learned from the examples formally. 
\begin{lemma}[Noise model]
Let a distribution $p(\x,\y,\rt,\BS,\bm\epsilon)$ be generated by $\y=F_{\rt}(\BS)+\bm\epsilon$ where $F_t$ is a function and $\BS,\bm\epsilon$ are random variables, we have the following.

\renewcommand{\labelenumi}{\arabic{enumi})}
\begin{enumerate}
\def\theenumi{\arabic{enumi})}
\item \hspace{-3mm} \footnote{See Appendix for variations, which are used in Theorem \ref{id_tion}.} \hspace{-2mm} $\BS$ is a B-score with $\F_t=F_t+G$ \emph{iff}.~$\E(\bm\epsilon|\BS,\x)=\E(\bm\epsilon|\BS,\rt)\coloneqq G(\BS)$.
\item 
Add
$\bm\epsilon\independent \BS|\x,\rt$ to the setting.
If there exists another $p'(\x,\y,\rt,\BS,\bm\epsilon)$ generated by $\y'=F'_{\rt}(\BS')+\bm\epsilon'$ such that
$p'(\y|\x,\rt)=p(\y|\x,\rt)$ 
(denoted by $\y\overset{d}{=}\y'|_{\x,\rt}$) 
and $\bm\epsilon\overset{d}{=}\bm\epsilon'|_{\x,\rt}$, then $F'_{\rt}(\BS')\overset{d}{=}F_{\rt}(\BS)|_{\x,\rt}$. 
\item \label{inj_model} \emph{\textbf{(Injective outcome model).}} Assume further that $F_{t}$ and $F'_t$ are injective. $p'(\x,\y,\rt,\BS,\bm\epsilon)$ as in 2). 
    \begin{enumerate}
    \item $\BS'\overset{d}{=}D'_t(\BS)|_{\x,\rt}$ where $D'_t\coloneqq{F'}_{t}\inv\circ F_{t}$. 
    \item \emph{(Identity of CATE).} If $\BS$ and $\BS'$ are Bt-scores for $p(\y|\x,\rt)$ and $p'(\y|\x,\rt)$, respectively, and $G_t$'s are zeros, then 
    \begin{equation}
        \mu'_t(\vx)=\mu_t(\vx) \text{ for any } \vx \in \mathcal{X} \bigcap \mathcal{X}',
    \end{equation}
    where $\mathcal{X}$ denotes support of $\x$, $\mathcal{X}'$ similarly.
    \end{enumerate}
\end{enumerate}
\end{lemma}

The idea of injective outcome model lays the \textbf{second foundation} of our method. It is general, particularly about the noise. In fact, all the true data generating processes we discuss in this paper are special case of it.  


The following is our general identification result. In particular, conclusion 3) shows the CATEs given by our model are same as truth.
\begin{theorem}[Identification by model]
\label{id_tion}
Given the family $p_{{\vf},{\bm\lambda}}(\rvy,\rvz|\rvx,\rt)$ specified by \eqref{model_indep} and \eqref{model_param}, 
assume that the true data distribution $p^*(\rvx, \rvy, \rt,\rconf, \rve)$ satisfies 
\renewcommand{\labelenumi}{\roman{enumi})}
\begin{enumerate}
\def\theenumi{\roman{enumi})}
    \item \label{ass:gen} $\vf^*(\BS^*( \rvx,\rconf,\rve_b),\rt)+\rve_a=\rvy$ (see Example 1),
    \item \label{ass:gen_f} $\vf^*_t$ is injective,
    \item \label{ass:gen_n} $\rve_a$ is factorized Gaussian with zero-mean and $\bm\sigma_{a,\rt}$,
    
\hspace{-8mm} and our model satisfies 

    \item \label{ass:model_f} $\vf_t$ is injective,
    \item \label{ass:model_g} $\vg_t(\rvz)= \bm\sigma_{a,t}^2$,
    \item \label{ass:model_latent} $\rvz$ is not lower-dimensional than $\BS^*$.
\end{enumerate}

Then, if $p_{\bar{\vf},\bar{\bm\lambda}}(\rvy|\rvx,\rt)=p^*(\rvy|\rvx,\rt)$, we have
\renewcommand{\labelenumi}{\arabic{enumi})}
\begin{enumerate}
\def\theenumi{\arabic{enumi})}
\item $\z_{\bar{\bm\lambda}}|\rvx,\rt \sim p_{\bar{\bm\lambda}}$ and $\z_{\bar{\vf},\bar{\bm\lambda}}|\rvx,\rvy,\rt \sim p_{\bar{\vf},\bar{\bm\lambda}}$ are Bt-scores for $p_{\bar{\vf},\bar{\bm\lambda}}(\rvy|\rvx,\rt)=p^*(\rvy|\rvx,\rt)$;
\item $\z_{\bar{\bm\lambda}}\overset{d}{=}\vd_{\rt}(\BS\st)|_{\x,\rt}$, and $\z_{\bar{\vf},\bar{\bm\lambda}}\overset{d}{=}\vd_{\rt}(\BS\st)|_{\x,\y,\rt}$, where $\vd_t\coloneqq\bar{\vf}_t^{-1}\circ\vf^*_t$.
\item Let $\bar{\y}(t)\coloneqq\vf_t(\z)+\bm\epsilon_t,\bm\epsilon_t \sim \mathcal{N}(\bm 0, \diag(\bm\sigma_{a,t}))$ be \emph{counterfactual} outcomes of the decoder. Then 
\begin{equation*}
\begin{split}
    \bar{\mu}_t(\vx)&=\E(\E(\bar{\y}(t)|\z_{\bar{\bm\lambda}},\vx))=\E(\bar{\vf}_t(\z_{\bar{\bm\lambda}})|\vx,\rt=t), \\
    \bar{\mu}_t(\vx)&=\E(\bar{\vf}_t(\z_{\bar{\vf},\bar{\bm\lambda}})|\vx,\rt=t), \\
    \mu_t(\vx)&=\bar{\mu}_t(\vx).
\end{split}
\end{equation*}

\end{enumerate}
\end{theorem}

\textit{Intuitions.} 
We confirm that both the true generating process and our decoder satisfy \ref{inj_model} in Lemma 1. So, conclusion 2) and 3) are applications of (a) and (b) in Lemma 1. More intuitions follows.
$\BS^*$ is a P-score of $(\rvx,\rconf,\rve_b)$ (compare Example 1), though in the proof we only need that it is a Bt-score. Also,
the $\z$'s
are Bt-scores for our generating model (the decoder). And the co-injectivity ensures the distributions of $\z$'s are the same as $\BS^*$ up to injective $\vd_t$. Due to the similar forms
of true generating process and our decoder, fitting model to the truth, the $\z$'s are also able to identify CATE for the true distribution. Finally, \ref{ass:model_latent} ensures $\z$ can contain all the useful information of $\BS\st$. In practice, we can just use higher dimensional $\z$ that is computationally tractable, as supported by our experiments. 

Note that, \ref{ass:gen_n} and \ref{ass:model_g} about the noise are \textit{not} critical for identification. We need them mainly to satisfy the ``noise matching'' in 2) of Lemma 1. Particularly, the Gaussians in the assumptions are readily relaxed, because, the outcome distribution in our decoder can be non-Gaussian, and the inference is still tractable. 

In fact, $\bm\epsilon\independent \BS|\x,\rt$ in 2) of Lemma 1 suggests us to learn noise model $\vg_t(\z)$, since there is no conditioning on $\x$ in our decoder. Desirably, $\vg$ obtains the relevant information of $\x$ through $\z$ by learning. And our experiments support this. We conjecture that an extended identifiability will show $\vg$ is learnable (though, following iVAE, $\vg$ is fixed in our current Theorem \ref{idmodel}).

According to Theorem \ref{id_tion}, our conditional prior also recovers a Bt-score, but the approximate posterior from encoder removes more uncertainty on $\z$, and is better under finite sample. As in previous, there may exist settings violating the noise assumptions of Theorem \ref{id_tion}, but we still give good estimation depending on $\x$ and $\y$.

\subsection{Estimation}
\label{sec:estimation}
We next discuss learning parameters that give the true observational distribution, and, more importantly, calculation of the latent representation $\z$ that is the Bt-score. Both are archived by the consistency of VAE, which can be based on some general assumptions

Consistent estimation of CATE follows directly from the identification. The following is a corollary of Theorem \ref{id_tion}.
\begin{corollary}[Estimation by VAE]
\label{th:estimation}

In Theorem \ref{id_tion}, 
further assume the consistency of Intact-VAE \eqref{model_indep}--\eqref{eq:enc}, which implies $p_{\vtheta'}(\rvz|\rvx,\rvy,\rt)=q_{\bm\phi'}(\rvz|\rvx,\rvy,\rt)$ where $\vtheta',\bm\phi'$ are the optimal parameters learned by the VAE.
Then, in the limit of infinite data $\mathcal{D}\coloneqq\{(\x,\y,\rt)\}\sim p\st(\x,\y,\rt)$, $\z_{\bm\phi',\rt} \sim  q_{\bm\phi'}(\rvz|\rvx,\rvy,\rt)$
is a Bt-score, and
\begin{equation}
\label{estimator}
    \mu_{t}(\vx)=\E_{\mathcal{D}|\vx,t}(\vf'_t(\vz_{\bm\phi',t})) 
\end{equation}


\end{corollary}
This estimator is highly nontrivial because it works under \textit{unobserved} confounder $\rconf$. In essence, we recovered a Bt-score containing sufficient information of $\rconf$ for identification, in $\z$. And $\z$ is calculated by our encoder. 

On the other hand, as mentioned in Introduction, a whole line of work aims to design better, balanced, estimator with observed confounding. Recall that the main problem of naive regression (e.g., \eqref{eq:id} would be naive if $\rconf$ was observed) is imbalance. If $p\st(\x|t)$'s are very different for some $\vx$, then we have few data points for one of $\mathcal{D}|\vx,t\coloneqq\{(\vx,\y,t)\}$, resulting in poor estimation. The estimator \eqref{estimator} also addresses imbalance to some extent, by learning a representation that is \textit{lower} dimensional than $\x$ (see also \citep{d2020overlap}), as Bt-scores often are. 

We take this idea further. With our next result, the sample size of \textit{each} treatment group is always the same as that of \textit{whole} dataset, regardless of covariate distribution. Based on our model identifiability (Theorem \ref{idmodel} in Appendix), we give a balanced estimator. In the proof, we require $\BS^*$ is a only B-score.
\begin{theorem}[Balanced Estimation]
\label{th:bestimation}

\begin{inparaenum}[i]
In addtion to the assumptions of Corollary \ref{th:estimation}, further assume 
\renewcommand{\labelenumi}{\roman{enumi})}
\def\theenumi{\roman{enumi})}
\item $\vf_t$ is differentiable, and
\item \label{ass:bcovar} $\blambda_0(\x)=\blambda_1(\x)$.
Then, 
$\z_{\bm\phi',0},\z_{\bm\phi',1}$
are B-scores that have the same distribution, and
\begin{equation}
\label{bestimator}
\begin{split}
    \mu_{\hat{t}}(\vx)&=\E_{\mathcal{D}|\vx}(\vf'_{\hat{t}}(\vz_{\bm\phi',t})), \hat{t} \in \{0,1\}.   
\end{split}
\end{equation}
\end{inparaenum}
\end{theorem}
Assumption \ref{ass:bcovar} is the key, and ensures learning B-scores, not only Bt-scores. It introduces stronger model identifiability than iVAE. i) is a technical assumption inherited from iVAE\footnote{We also omit technical assumption iii) of Theorem \ref{idmodel}, which is not very relevant, and our general balancing assumption implies it.}. 

\ref{ass:bcovar} adds balancing into our estimator. The prior $p_{\bm\lambda}(\rvz|\rvx,\rt)=p_{\bm\lambda}(\rvz|\rvx)$ is unconditional, independent of $\rt$ given $\rvx$. Just like balancing score gives balanced estimator, with B-score, we have correspondence between \eqref{bestimator} and \eqref{eq:cate_by_bs} where $p(\BS|\x)$ is the score distribution, in addition to that between \eqref{estimator} and \eqref{eq:cate_by_bts}.
The same prior for the treatment groups, i.e., the \textit{balanced} prior of latent representation, is similar to balanced representation learning \citep{johansson2016learning, shalit2017estimating}, where balanced representation is favored by ad hoc regularization. This is also related to the fact that, when building CVAE, unconditional prior can achieve better performance \citep{kingma2014semi}. Actually, \ref{ass:bcovar} is an extreme case of a very general but technically involved assumption we present in Appendix, which is a natural form of balancing.

Our overall \textbf{algorithm} steps should be clear. \textit{After} training Intact-VAE, we feed data
into the encoder $q(\rvz'|\rvx=\vx,\rvy=y,\rt=t)$, and draw posterior sample from it. 
Then, we follow \eqref{bestimator} closely. Setting $\rt=\hat{t}$ in the decoder, feed the posterior sample into it, 
we get counterfactual sample $\vy'(\hat{t})$ as output of the decoder.
Finally, we estimate ATE by taking average $\E_{\mathcal{D}}(\vy'(1)-\vy'(0))$, and CATE by $\E_{\mathcal{D}|\vx}(\vy'(1)-\vy'(0))$, adding conditioning on $\vx$.

As mentioned, by taking $\y$, posterior model (the encoder) is better than conditional prior.
On the other hand, sampling posterior requires \textit{post-treatment} observation $y$. Often, it is desirable that we can also have \textit{pre-treatment} prediction for a new subject, with only the observation of its covariate $\rvx=\vx$. To this end, we use conditional prior $p(\rvz'|\rvx)$ as a pre-treatment predictor for $\rvz'$: input $\vx$ and draw sample from $p(\rvz'|\rvx=\vx)$ instead of $q$, and all the others remain the same. We will also have sensible pre-treatment estimation of treatment effects, as ensured by Theorem \ref{id_tion}.

\section{Experiments}

We use the proposed Intact-VAE model for three types of data, and compare the results with existing methods.

Unless otherwise indicated, for VAE models we use a multilayer perceptron (MLP) that has 3*200 hidden units with ReLU activation, for each function $\vf,\vg,\vh,\vk,\vr,\vs$ in \eqref{model_param}\eqref{eq:enc}, and $\bm\lambda=(\vh,\vk)$ depends only on $\rvx$. 
Note that, while Theorem~\ref{id_tion} assumes the outcome noise $\vg$ is fixed and known, we train $\vg$ also.  
The Adam optimizer with initial learning rate $10^{-4}$ and batch size 100 is employed. More details on hyper-parameters and experimental settings are given in each experiment and Appendix.

All experiments use early-stopping of training by evaluating the ELBO on a validation set. We test post-treatment results on training and validation set jointly. This is non-trivial. Recall the fundamental problem of causal inference in Introduction and Sec.~\ref{sec:setup}. The treatment and (factual) outcome should not be observed for pre-treatment predictions, so we report them on a testing set (see the end of Sec.~\ref{sec:treatment}).

As in previous works \citep{shalit2017estimating, louizos2017causal}, we report the absolute error of ATE $\epsilon_a\coloneqq |\E_{\mathcal{D}}(y(1)-y(0)) - \E_{\mathcal{D}}(y'(1)-y'(0))|$, and the square root of empirical PEHE \citep{hill2011bayesian} $\epsilon_i\coloneqq \E_{\mathcal{D}}((y(1)-y(0))-(y'(1)-y'(0)))^2$ for individual-level treatment effects.

\subsection{Synthetic Dataset}

We generate synthetic datasets following \eqref{art_model}. For variations (see Appendix), we introduce three different causal settings: unobserved confounder $\rz$, IV $\x$, and unconfounded (conf., inst., and ig., respectively, in Figure \ref{nonl_art}). $\mu_i$ and $\sigma_i$ are randomly generated. The functions $h,k,l$ are linear with random coefficients. And $f_0,f_1$ is built by separated NNs.
We generate linear and nonlinear (invertible) outcome models, and set the outcome and proxy noise level by $\alpha,\beta$ respectively. See Appendix for more details.

\begin{equation}
\label{art_model}
\begin{split}
\textstyle
    \rvx \sim \prod_{i=1}^3 \mathcal{N}(\mu_i, \sigma_i); \quad
    \rz|\rvx \sim \mathcal{N}(h(\rvx), \beta k(\rvx)); \\
    \rt|\rvx,\rz \sim \bern(\logi(l(\rvx,\rz))); \\
    \rvy|\rz,\rt \sim \mathcal{N}(C_{\rt}^{-1}f(\rz,\rt), \alpha). 
\end{split}
\end{equation}

\begingroup

\setlength{\columnsep}{8pt}
\setlength{\intextsep}{5pt}
\begin{wrapfigure}{r}{0.25\textwidth}
\vspace{-.2in}
  \begin{center}
    \includegraphics[width=0.25\textwidth]{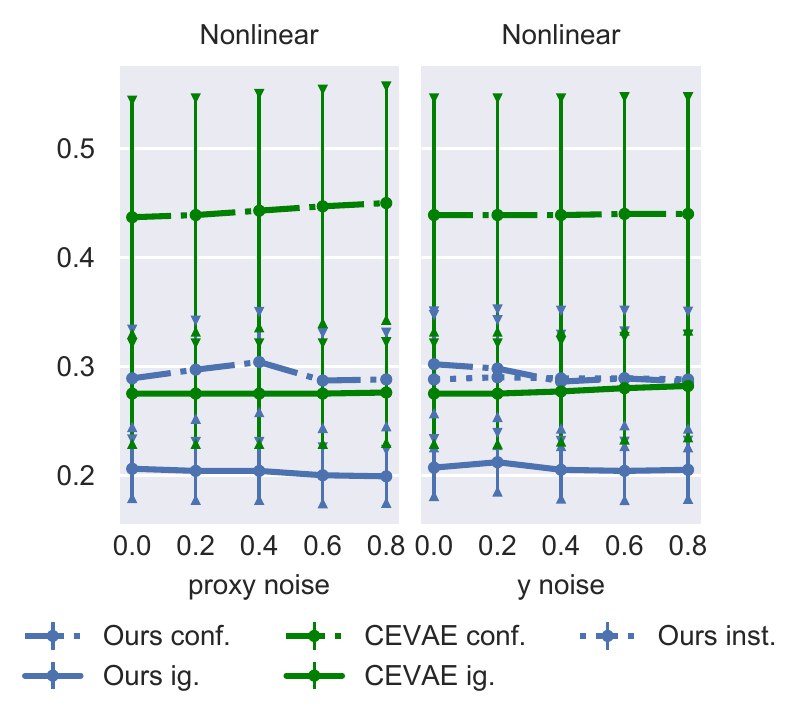}
  \end{center}
  \vspace{-.3in}
  \caption{Pre-treatment \footnotesize{$\sqrt{\epsilon_i}$ on nonlinear synthetic dataset. Error bar on 100 random models. We adjust one of the noise levels $\alpha,\beta$ in each panel, with another fixed to 0.2. See Appendix for results on linear outcome. Results for ATE and post-treatment are similar.}}
\vspace{-.05in}
\label{nonl_art}
\end{wrapfigure}


In each causal setting, and with the same kind of outcome models, and the same noise levels ($\alpha,\beta$), we evaluate Intact-VAE and CEVAE on 100 random data generating models, with different sets of functions $f,h,k,l$ in \eqref{art_model}. For each model, we sample 1500 data points, and split them into 3 equal sets for training, validation, and testing. Both the methods use 1-dimensional latent variable in VAE. For fair comparison, all the hyper-parameters, including type and size of NNs, learning rate, and batch size, are the same for both the methods. 

Figure \ref{nonl_art} shows our method significantly outperforms CEVAE on all cases; CEVAE does not use conditional prior and has no theoretical guarantee in the current setting. Both methods work the best under unconfoundedness (``ig.''), as expected. The performances of our method on IV (``inst.'') and proxy (``conf.'') settings match that of CEVAE under ignorability, showing the effective deconfounding.

\setlength{\columnsep}{8pt}
\setlength{\intextsep}{5pt}
\begin{wrapfigure}{r}{0.25\textwidth}
\vspace{-.2in}
  \begin{center}
    \includegraphics[width=0.25\textwidth]{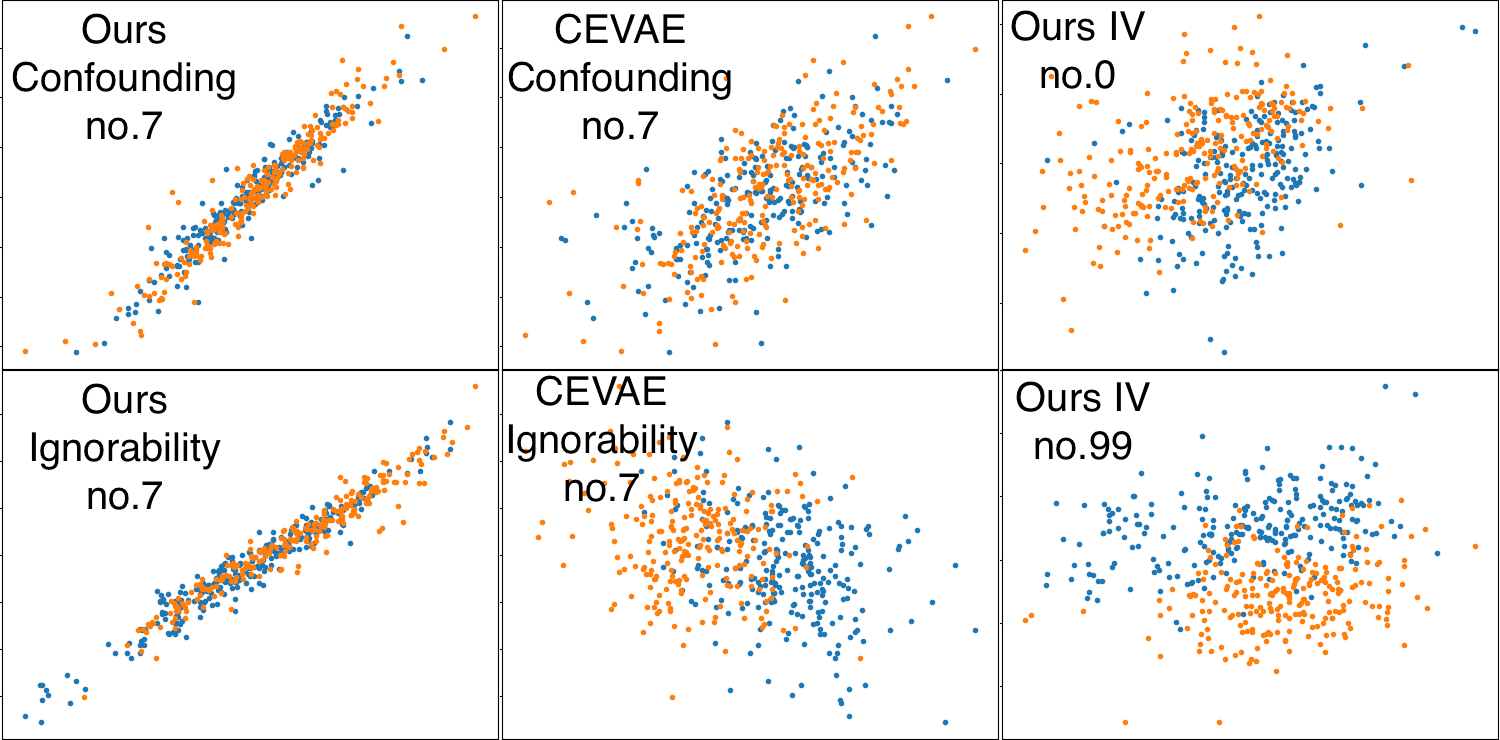}
  \end{center}
   \vspace{-.2in}
 
  \caption{\footnotesize{Plots of recovered 
  - true 
  latent on the nonlinear outcome. Blue: $t=0$, Orange: $t=1$. $\alpha,\beta= 0.4$. ``no.'' indicates index among the 100 random models.}} 
\vspace{-.1in}
\label{recover}
\end{wrapfigure}

Here, the true latent $\rz$ is a B-score. 
And there are no better candidate B-scores than $\rz$, because $f_t$ is invertible and no information can be dropped from $\rz$. 
Thus, as shown in Figure \ref{recover}, our method learns representation as an approximate affine transformation of the true latent value, as a result of our model identifiability. As expected, both recovery and estimation are better with \textit{unconditional} prior $p(\rz'|\rvx)$, and we can see an example of bad recovery using conditional $p(\rz'|\rvx,\rt)$ in Appendix. CEVAE shows much lower quality of recovery, particularly with large noises. Under IV setting, while treatment effects are estimated as well as for confounding, the relationship to the true latent is significantly obscured, because the true latent is correlated to IV $\x$ only given $\rt$, while we model it by $p(\rz'|\rvx)$. This experimentally confirms that our method does not need to recover the true score distribution.

We can see our method is robust to the unknown noise level.
This indicates that noises are learned by our VAE. 
We can see in Appendix that the noise level affects how well we recover the latent variable. 

\endgroup

\subsection{IHDP Benchmark Dataset}
This experiment shows our balanced estimator matches the state-of-the-art methods specialized for unconfoundedness. The IHDP dataset \citep{hill2011bayesian} is widely used to evaluate machine learning based causal inference methods, e.g. \cite{shalit2017estimating, shi2019adapting}. Here, \textit{unconfoundedness} holds given the covariate, and thus the covariate is just a B-score. However, conditioning on covariate is sufficient but not necessary, and the necessary B-score is a linear combination of the covariates. See Appendix for details.

As shown in Table \ref{t:ihdp}, Intact-VAE outperforms or matches the state-of-the-art methods. To see our balancing property clearly, we add two components specialized for balancing from \citet{shalit2017estimating} into our method (whose results is  shown in the caption of Table \ref{t:ihdp}), and compare to our unmodified estimator. First, we build the two outcome functions $\vf_t(\rvz),t=0,1$ in our learning model \eqref{model_param}, using two separate NNs. Second, we add to our ELBO \eqref{elbo} a regularization term, which is the Wasserstein distance \citep{cuturi2013sinkhorn} between the learned $p(\rvz'|\rvx,\rt=0)$ and $p(\rvz'|\rvx,\rt=1)$.

In particular, our method has the \textit{best} ATE estimation \textit{without} the additional components; and it has the \textit{best} individual-level estimation, adding the two components from \cite{shalit2017estimating}. 
We can see in the caption of Table \ref{t:ihdp}, the specialized additions do \textit{not} really improve our method, only causing a tradeoff between CATE and ATE estimation, and this may due to the tradeoff between fitting and balancing.  
And notably, our method outperforms other generative models (CEVAE and GANITE) by large margins.

We find higher than 1-dimensional latent variable in Intact-VAE gives better results, because we have \textit{discrete} true B-score due to the existence of discrete covariates.
We report results with 10-dimensional latent variable. The robustness of VAE under model misspecification was also observed by \citet{louizos2017causal}, where they used 5-dimensional Gaussian latent variable to model a binary ground truth.

\begin{table}[h]

\centering
\scriptsize
\caption{Errors on IHDP. The mean and std are calculated over 1000 random data generating models. *Results with modifications are $\epsilon_a=.31_{\pm .01}/.30_{\pm .01}$ and $\sqrt{\epsilon_i}=\textbf{.77}_{\pm .02}/\textbf{.69}_{\pm .02}$. \textbf{Bold} indicates method(s) that are \textit{significantly} better than all the others. The results of the others are taken from \cite{shalit2017estimating}, except GANITE \citep{yoon2018ganite} and CEVAE \citep{louizos2017causal}.} 

\begin{tabular}{p{1.cm}p{.5cm}p{.5cm}p{.5cm}p{.5cm}p{.5cm}p{.6cm}p{.6cm}}
\toprule 
Method &{TMLE} &{BNN} &{CFR} &{CF} &{CEVAE} &{GANITE} &{Ours*}\\
\midrule 
pre-$\epsilon_a$ &NA &.42$_{\pm .03}$
&.27$_{\pm.01}$ &.40$_{\pm.03}$ &.46$_{\pm.02}$ &.49$_{\pm.05}$  &\textbf{.21}$_{\pm .01}$
\\
\midrule 
post-$\epsilon_a$ &.30$_{\pm .01}$ &.37$_{\pm .03}$  
&.25$_{\pm.01}$  &\textbf{.18}$_{\pm.01}$ &.34$_{\pm.01}$   &.43$_{\pm.05}$  &\textbf{.17}$_{\pm .01}$
\\
\midrule 
pre-$\sqrt{\epsilon_i}$ &NA &2.1$_{\pm .1}$ 
&\textbf{.76}$_{\pm.02}$ &3.8$_{\pm.2}$ &2.6$_{\pm.1}$ &2.4$_{\pm .4}$ &1.0$_{\pm .05}$
\\
\midrule 
post-$\sqrt{\epsilon_i}$ &5.0$_{\pm .2}$ &2.2$_{\pm .1}$  
&\textbf{.71}$_{\pm.02}$ &3.8$_{\pm.2}$ &2.7$_{\pm.1 }$  &1.9$_{\pm .4}$ &.97$_{\pm .04}$
\\
\bottomrule 
\end{tabular}
\label{t:ihdp}
\end{table}

\subsection{Pokec Social Network Dataset}
We show our method is the best compared with the methods specialized for networked deconfounding, a challenging problem in its own right. Pokec \citep{leskovec2014snap} is a real world social network dataset. We experiment on a semi-synthetic dataset based on Pokec, which was introduced in \citet{veitch2019using}, and use exactly the same pre-processing and generating procedure. The pre-processed network has about 79,000 vertexes (users) connected by 1.3 $\times 10^6$ undirected edges.  The subset of users used here are restricted to three living districts that are within the same region. The network structure is expressed by binary adjacency matrix $\mG$. Following \citet{veitch2019using}, we split the users into 10 folds, test on each fold and report the mean and std of pre-treatment ATE predictions. We further separate the rest of users (in the other 9 folds) by $6:3$, for training and validation. 

Each user has 12 attributes, among which  \texttt{district}, \texttt{age}, or \texttt{join date} is used as a confounder $\rz$ to build 3 different datasets, with remaining 11 attributes used as covariate $\rvx$. Treatment $\rt$ and outcome $\rvy$ are synthesised as following:
\begin{equation}
\label{pokec}
    \rt \sim \bern(g(\rz)), \quad \rvy = \rt + 10(g(\rz)-0.5) + \epsilon,
\end{equation}
where $\epsilon$ is standard normal. Note that \texttt{district} is of 3 categories; \texttt{age} and \texttt{join date} are also discretized into three bins. $g(\rz)$, which is a B-score, maps these three categories and values to $\{0.15, 0.5, 0.85\}$. 

Intact-VAE is expected to learn a B-score from $\mG, \rvx$, if we can exploit the network structure effectively. 
Given the huge network structure, most users can practically be identified by their attributes and neighborhood structure, which means $\rz$ can be roughly seen as a deterministic function of $\mG, \rvx$. 
This idea is comparable to Assumptions 2 and 4 in \citet{veitch2019using}, which postulate directly that a balancing score can be learned in the limit of infinite large network. 
To extract information from the network structure, we use Graph Convolutional Network (GCN) \citep{DBLP:conf/iclr/KipfW17} in conditional prior and encoder of Intact-VAE. The implementation details are given in Appendix.

Table \ref{tbl:Pokec} shows the results. We report pre-treatment PEHE of our method in the Appendix, while \cite{veitch2019using} does not give individual-level prediction.

\begin{table}[h]
\centering
\scriptsize

\caption{Pre-treatment ATE on Pokec. Ground truth ATE is 1, as we can see in \eqref{pokec}. ``Unadjusted'' estimates ATE by $\E_{\mathcal{D}}(y_1)-\E_{\mathcal{D}}(y_0)$. ``Parametric'' is a stochastic block model for networked data \citep{gopalan2013efficient}. ``Embed-'' denotes the best alternatives given by \citet{veitch2019using}. \textbf{Bold} indicates method(s) that are \textit{significantly} better than all the others. 20-dimensional latent variable in Intact-VAE works better, and its result is reported. The results of the other methods are taken from \citet{veitch2019using}.} \label{tbl:Pokec}

\begin{tabular}{l|c|c|c}
\toprule 
              & \texttt{Age}   & \texttt{District} & \texttt{Join Date} \\
              \midrule
Unadjusted    & 4.34 $\pm$ 0.05  & 4.51 $\pm$ 0.05  & 4.03 $\pm$ 0.06   \\
Parametric  & 4.06 $\pm$ 0.01   & 3.22 $\pm$ 0.01  & 3.73 $\pm$ 0.01   \\
Embedding-Reg.  & 2.77 $\pm$ 0.35  & \textbf{1.75} $\pm$ 0.20   & 2.41 $\pm$ 0.45   \\
Embedding-IPW  & 3.12 $\pm$ 0.06  & \textbf{1.66} $\pm$ 0.07   & 3.10 $\pm$ 0.07   \\
Ours           &\textbf{2.08} $\pm$ 0.32   & \textbf{1.68} $\pm$ 0.10   & \textbf{1.70} $\pm$ 0.13  \\
\bottomrule
\end{tabular}
\end{table}


Finally, we note that, in all the experiments, learned noise at least matches fixed one, and sometimes significantly better. Also refer to the end of Sec.~\ref{sec:identification} for rationale. 

It is also noteworthy that, both IHDP and Pokec semi-synthetic datasets are special cases of the injective noise model, and both are commonly used in previous works. This shows the significance and generality of our theory. And the comparisons are fair: in implementation, our method does \textit{not} enforce learning injective outcome models (e.g., by normalizing flows \citep{kobyzev2020normalizing}).

\section{Conclusion}

In this work, we proposed a new VAE architecture for estimating causal effects under unobserved confounding, with theoretical analysis and state-of-the-art performance. To the best of our knowledge, this is the \textit{first} generative learning method that provably identifies treatment effects, without assuming that the hidden confounder can be recovered. Following the line of sufficient scores in causal inference, we introduced B*-scores, which is new to machine learning methods, together with the related mean exchangeability \citep{dahabreh2019generalizing}. The the generality and properties of injective outcome model is another key to our method, as we showed in Lemma 1 (Sec.~\ref{sec:identification}) and Experiments.

We emphasize that, recovery of the true distribution of Bt-score is also \textit{not} needed, because Bt-score is based on conditional mean and mismatched score distribution is allowed by our theory. This is why we saw in experiments that Gaussian latent variable works well. On the other hand, our method can be easily extended to learn arbitrary observational and latent distributions, because our theory extends to exponential family distributions of latent variables (see Theorem \ref{idmodel} in Appendix, also \citep{khemakhem2020variational}), and VAE can use non-Gaussian latent \citep{maddison2016concrete}.
More discussions can be found in Appendix.

\bibliography{example_paper}
\bibliographystyle{icml2021}

\clearpage

\section{Proofs}
Here, we give proofs of results presented in the main text. More theoretical results and their proof can be found in ``Theoretical exposition''.

By a slight abuse of symbol, we will overload $|$ to collect the dependence on $\rt=t$, e.g., \eqref{estimator} can be written as $    \mu(\vx)=\E_{\mathcal{D}|\vx}(\vf'(\vz_{\bm\phi'}))|t$. 

Theorem \ref{cate_by_bts} and Corollary \ref{cate_by_bs} are rather straightforward from \eqref{eq:ps_indp} (see Proposition \ref{prop:pscore}) and the definitions of P*/B*-scores, and thus the proofs are omitted.

\textbf{Note on notation related to treatment assignment.} We'd better make it clear that, in this paper, when specifying a generating process, if the dependence on $\rt$ is written explicitly (often in subscripts, e.g., $F_{\rt}$ in Lemma 1 and $\z_{\bphi',\rt}$ in Corollary \ref{th:estimation}), then it is meant to be causally affected by the treatment assignment. Otherwise, the random variables may correlate to $\rt$, but are not affected by treatment assignment (e.g., $\BS$ and $\beps$ in Lemma 1). Please have a look at the comments below the proof of Lemma 1, particularly when footnotes in the proof of Theorem \ref{id_tion} are not clear to you.
\begin{proof}[Proof of Lemma 1]
1)
By definition of the generating process, $\y(t)= F_t(\BS)+\beps,t \in \{0,1\}$ (note that consistency of counterfactuals should be satisfied, and treatment assignment only affects $\y$ though $F$). Plug $\y(t)$ into the definition of B0/B1/B-score, we quickly see the results. We should note that,
\begin{equation}
\begin{split}
    \E(\bm\epsilon|\BS,\x)&=\E(\bm\epsilon|\BS,\rt) \iff \\ \E(\bm\epsilon|\BS,\x)&=\E(\bm\epsilon|\BS,\rt=t) \text{ for \textit{both} } t,
\end{split}
\end{equation}
because $\beps$ is not affected by treatment assignment. Also, given $\BS$ is not affected by treatment assignment, $\BS$ is a B-score iff.~when it is a Bt-score. See the comments below this proof, about what happens when $\beps$ or $\BS$ are affected by treatment assignment.


2) Due to the assumptions, if $\x,\rt$ is given, the generating process can be seen as a mean $F_{t}(\BS)$ plus a noise whose distribution does \textit{not} change with the mean. Now, the distribution $p(\y|\x,\rt)$ can be thought as following: first we have $p(\BS|\x,\rt)$, then $\E(\y(t)|\x,\rt=t)$ is \textit{determined} totally by $F_{t}(\BS)$, and $\y(t)$ is the former plus fixed noise (that might not be zero-mean). Thus, due to the \textit{fixed} noise, to have $p'(\y|\x,\rt)=p(\y|\x,\rt)$, the densities of the means (seen as two r.v.s) should match exactly at every point. 

\ref{inj_model} 

(a) is obvious from part 2). 

(b) Use eq.~\eqref{eq:cate_by_bts}, and note $\F_t(\BS_t)=\E(\y(t)|\BS_t,\rt=t)$, we have
\begin{equation}
\begin{split}
    \mu_t(\vx)&=\textstyle\int \F_t(B_t)p(B_{t}|\vx)dB_{t} \\
    &=\textstyle\int_{\BS} F_t(B)p(B|\vx,t)dB
\end{split}
\end{equation}
First equality uses that $\BS$ is a Bt-score, and second equality applies $\F=F$ and change of variable $\BS=\BS_t\text{ if }\rt=t$.
For the same reason,
\begin{equation}
    \mu'_t(\vx)=\textstyle\int_{\BS'} F'_t(B)p'(B|\vx,t)dB.
\end{equation}
Apply the change of variable $B'=D'_t(B)$ from (a) to (\theequation), we have the result.
\end{proof}
In 1) of Lemma 1, if $\beps$ is affected by treatment assignment, then, borrow the notation $\y(t)$ to $\beps$, we have $\y(t)= F_t(\BS)+\beps(t)$, and
\begin{equation}
\label{eq:cf_noise}
    \begin{split}
    \E(\bm\epsilon|\BS,\x)&=\E(\bm\epsilon|\BS,\rt) \implies \\ \E(\bm\epsilon(t)|\BS,\x)&=\E(\bm\epsilon(t)|\BS,\rt=t) \text{ for \textit{both} } t,
\end{split}
\end{equation}
but not the opposite. Still, $\BS$ is a B-score given either one of the conditions.

However, if $\BS$ is affected by treatment assignment, then we have $\BS_t$ if $\rt=t$, two different random variables w.r.t.~two treatment assignments. Then, under the condition
\begin{equation}
\label{eq:cf_bscore}
    \E(\bm\epsilon|\BS_t,\x)=\E(\bm\epsilon|\BS_t,\rt=t) \text{ for \textit{both} } t,
\end{equation}
we only have a \textit{Bt-score}. 

As we can see in Theorem \ref{id_tion} and its proof, both the noises in truth generating process and our decoder are causally affected by treatment assignment (see footnote \ref{ft:noise}). And, both our conditional prior and posterior do involve treatment assignment, thus, both only give Bt-scores (see footnote \ref{ft:z}).

\begin{proof}[Proof of Theorem \ref{id_tion}]

We first prove the results for $\z_{\bar{\bm\lambda}}$, then extend the results to $\z_{\bar{\vf},\bar{\bm\lambda}}$.

We check the assumptions for \ref{inj_model} in Lemma 1. 

First, co-injectivity is given. 

In the true generating process, we have $\BS\st$ is a B-score (by 1) of Lemma 1)\footnote{\label{ft:noise}Not exactly, see also the comment around \eqref{eq:cf_noise}. Note, intuitively, when specifying assumption \ref{ass:gen_n}, our intention is \textit{causal}. That is, we \textit{assign} $\rt=t$ and $\bm\sigma_{a,t}$ as variance of $\rve_a$, and, we \textit{could have} assign $\rt=1-t$ and $\bm\sigma_{a,1-t}$. In fact, we have $\rve_{a,t} \independent \rvx,\rconf,\rve_b|\BS\st$ from the exogeneity of $\rve_a$ given $\rt$, thus $\BS\st$ is also a \textit{P-score} of $(\rvx,\rconf,\rve_b)$. Compare to Example 1, here the required independence is given by conditioning on $\rt=t$, which is a result of treatment assignment.}, since $\rve_a$ is zero-mean. We have $\rve_a \independent \BS\st|\x,\rt$ since $\rve_a$ is exogenous (independent of any variables under study) \textit{given} $\rt$.

Similarly for our decoder, we check $\beps$ and get the conclusions for $\z_{\bar{\bm\lambda}}$. So, we have conclusion 1) for $\z_{\bar{\bm\lambda}}$, since it is a Bt-score\footnote{\label{ft:z}Again, by a variation of 1) of Lemma 1, see also the comment around \eqref{eq:cf_bscore}. We can also see it as a \textit{Pt-score} of $(\x,\beps \sim \mathcal{N}(\bm0,\mI))$. Note again, we should see the conditional prior as the generating process $\rz_{i,\bar{\lambda}_i}=h_{i,t}(\x)+\sqrt{k_{i,t}(\x)}\epsilon_i$, which involves the assignment $\rt=t$.} for the decoder. Also, we have conclusion 2) for $\z_{\bar{\bm\lambda}}$, since $\beps \deq \bm\sigma_{a,\rt}\xt$ by \ref{ass:gen_n} and \ref{ass:model_g} and $\bar{\y}\deq\y\st\xt$ by the ``if'' part of conclusions.

Now that we have the conditions for \ref{inj_model} in Lemma 1, for conclusion 3), we only need to derive the equation for $\bar{\mu_t}(\vx)$. This is by a direct application of \eqref{eq:cate_by_bts}. We should pay attention to the counterfactuals defined for the decoder, given in 3). Note also the support of $\x$ is matched between our model and the truth, by assumption.

We extend the result to $\z_{\bar{\vf},\bar{\bm\lambda}} \sim  p_{\bar{\vf},\bar{\bm\lambda}}(\rvz|\rvx,\rvy,\rt)$. We follow the reasoning of Lemma 1. Note that, further given $\y=\vy$, the \textit{means} of $\y(t),\bar{\y}(t)$ should still be identically distributed (i.d) given $\x,\rt$, again because the \textit{distribution} of noise (in decoder and truth) does not depend on $\y$. Thus, we have, given any $(\vx,\vy,t)$ (add condition on $\y$), the \textit{posterior} means of $\y(t),\bar{\y}(t)$ should be i.d. That is, $\bar{\vf}_{\rt}(\z_{\bar{\vf},\bar{\bm\lambda}})\overset{d}{=}\vf\st_{\rt}(\BS\st)|_{\x,\y,\rt}$. The rest is the same as the reasoning for conditional prior $\z_{\bar{\bm\lambda}}$.
To understand $\z_{\bar{\vf},\bar{\bm\lambda}}$ as a Bt-score, it might be helpful to imagine $\z_{\bar{\vf},\bar{\bm\lambda}}$ was generated from $\x,\y,\rt$, and then input into our decoder (which, as you may know, is done through variational inference in the ELBO, particularly our encoder). 
Note also that, the support of $\z_{\bar{\vf},\bar{\bm\lambda}}$ is contained in that of $\z_{\bar{\bm\lambda}}$, due to the additional information given by $\y$.
\end{proof}

The definition of counterfactuals for decoder deserves more words. It introduces intervention (i.e., assignment of treatment value) into our model, because we also have consistent counterfacuals in the model: $\bar{\rvy}(t)\coloneqq \vf_t(\z)+\beps=\bar{\y}$ if $\rt=t$. And this is important to understand Theorem \ref{th:bestimation}, where we will assign \textit{counterfactual} treatment value, that may not be the same as the value of $\rt$.


In estimation, we need \textit{consistency}\footnote{This is the statistical consistency of an estimator. Do not confuse with the consistency of counterfactuals.} of our VAE (Proposition \ref{consistency}), to learn an observational distribution same as truth in the limit of infinite data, and to have $\z_{\vtheta',\rt}=\z_{\bm\phi',\rt}$. 
\begin{proof}[Proof of Corollary \ref{th:estimation}]
From Theorem \ref{id_tion} and the consistency of VAE (Proposition \ref{consistency}), we have $\z_{\vtheta',\rt}=\z_{\bm\phi',\rt}$ is a Bt-score, and $\vf'_t(\z_{\bm\phi',t})=\E(\rvy(t)|\z_{\bm\phi',t},\rt=t)$, so we have \eqref{estimator}. 
\end{proof}

For Theorem \ref{th:bestimation}, it is important that the learned parameters $\vtheta'$ is identified in the equivalence class defined by \eqref{eq:class} in Theorem \ref{idmodel} (model identifiability).
\begin{proof}[Proof of Theorem \ref{th:bestimation}]
As an expository step, since $\z_{\blambda',0}\deq\z_{\blambda',1}\deq\vd(\BS\st)|_{\x}$ trivially from \ref{ass:bcovar} and Theorem \ref{id_tion}, $\z_{\blambda'}$ becomes a \textit{B-score}, and
we have
\begin{equation}
\begin{split}
    \mu_{1-t}(\vx)&=\textstyle\int_{\BS} F_{1-t}(B)p(B|\vx,{1-t})dB \\
    &=\textstyle\int_{\vz} \vf'_{1-t}(\vz)p_{\blambda'}(\vz|\vx,1-t)d\vz \\
    &=\textstyle\int_{\vz} \vf'_{1-t}(\vz)p_{\blambda'}(\vz|\vx,t)d\vz=\mu'_{1-t}(\vx).
\end{split}
\end{equation} 
The second equality uses the same technique in the proof of \ref{inj_model} in Lemma 1, and the third uses $p_{\blambda'}(\z|\vx,0)=p_{\blambda'}(\z|\vx,1)$. 

Note, particularly, the third line of (\theequation) can be written in another way: $\mu'_{1-t}(\vx)=\E(\vf'_{1-t}(\z_{\blambda',t})|\vx)$, and understood as: the \textit{counterfactual} conditional outcome $\mu'_{1-t}(\vx)$ can be given by \textit{factual} distribution of score $p_{\blambda'}(\vz|\vx,t)$. 

We go on to prove, for posterior $\z_{\btheta'}$, we have similar results; $\z_{\btheta',0}\deq\z_{\btheta',1}\deq\vd(\BS\st)|_{\x,\y}$ are B-scores, and  $\mu'_{1-t}(\vx)=\E(\vf'_{1-t}(\z_{\btheta',t})|\vx)$, then our final goal is just a corollary of the consistency of VAE.

To consider the posterior $\z_{\btheta'}$, we need to use the identifiability of our model (Theorem \ref{idmodel}, next section)\footnote{We would better include iii) in Theorem \ref{idmodel} also in the statement of Theorem \ref{th:bestimation}, but we hope it would not be confusing without the inclusion since we stated clearly that Theorem \ref{th:bestimation} is based on Theorem \ref{idmodel}.}. Specifically, parameter $\vf_t$ is identified up to an affine transformation:
\begin{equation}
    \vf' = \vf \circ \mathcal{A} |t
\end{equation}
where $\vf,\vf'$ are any pair of optimal parameters satisfying $\y'\deq\y\deq\y\st\xt$. 

For $\vf=\vf^*$ in our model, there should exists $\bm\lambda=\bm\lambda^*$ such that $\z_{\vtheta,\rt} \sim \vaepostparam$ and $\BS^*$ have the same distribution, and $\vaeobsparam=\trueobs$. That is, $(\vf^*, \bm\lambda^*)$ should be in the set of optimal parameters. Thus,
\begin{equation}
    \vf' = \vf\st \circ \mathcal{A} |t
\end{equation}

In the proof of Theorem \ref{idmodel}, we can see $\mathcal{A}_t$ depends on $t$ only through $\blambda_t$. Then, $\mathcal{A}_0=\mathcal{A}_1$ trivially by $\blambda_0=\blambda_1$.
Thus, applying 2) of Theorem \ref{id_tion}, we have $\vd_0=\vd_1=\mathcal{A}\inv$, and $\z_{\btheta',0}\deq\z_{\btheta',1}\deq\mathcal{A}\inv(\BS\st)|_{\x,\y}$.
\end{proof}
In the proof, We see again the view that our decoder is a generating process with treatment assignment. We also see, to have $\mathcal{A}_0=\mathcal{A}_1$, it seems unnecessary to require $\blambda_0=\blambda_1$, though the latter is convenient and still widely applicable. Indeed, we derive a general balancing assumption in Sec.~\ref{bcovar}, of which $\blambda_0=\blambda_1$ is just a extreme case.

With balanced representation, in theory, we can also use in \eqref{bestimator} counterfactual representation that is the same as factual representation. However, it is safer to use factual representation, because it is possible in practice that balancing failed for some sample points and factual representation would give smaller error.

A final word about \textit{positivity}. With Theorem \ref{id_tion} and Corollary \ref{th:estimation}, we further need, and in practice we often have, the positivity of the $\z$'s ($p(\rt|\z)>0$ always) for identification and estimation of CATE for every $\vx \in \mathcal{X}$. Nevertheless, we will in fact use, as in our experiments, the balanced estimator in Theorem \ref{th:bestimation} where positivity is ensured for every $\vx \in \mathcal{X}$ since the $\z$'s are B-scores.

\section{Additional backgrounds}

\subsection{Prognostic score and balancing score}

In the fundamental work of \citet{hansen2008prognostic}, prognostic score is defined equivalently to our P0-score, but it in addition requires no effect modification to work for $\y(1)$. Thus, a useful prognostic score corresponds to our Pt-score. We give main properties of Pt-score as following. 
\begin{proposition}
\label{prop:pscore}
If $\rvv$ gives weak ignorability, and $\PS_{\rt}(\rvv)$ is a Pt-score, then $\rvy(t)\independent \rvv,\rt|\PS_t$.
\end{proposition}

The following three properties of conditional independence will be used repeatedly in proofs.
\begin{proposition}[Properties of conditional independence]
\label{indep_prop}
\citep[1.1.55]{pearl2009causality} For random variables $\rvw, \rvx, \rvy, \rvz$. We have:
\begin{equation*}
    \begin{split}
        \rvx \independent \rvy|\rvz \land \rvx \independent \rvw|\rvy, \rvz &\implies \rvx \independent \rvw,\rvy|\rvz \text{ (Contraction)}. \\
        \rvx \independent \rvw,\rvy|\rvz &\implies \rvx \independent \rvy|\rvw,\rvz \text{ (Weak union)}. \\ 
        \rvx \independent \rvw,\rvy|\rvz &\implies \rvx \independent \rvy|\rvz \text{ (Decomposition)}.
    \end{split}
\end{equation*}
\end{proposition}

\begin{proof}[Proof of Proposition \ref{prop:pscore}]
From $\rvy(t)\independent\rt|\rvv$ (\textit{weak ignorability} of $\rvv$), and since $\PS_t$ is a \textit{function} of $\rvv$, we have $\rvy(t)\independent\rt|\PS_t,\rvv$ (1).

From (1) and $\rvy(t)\independent\rvv|\PS_t(\rvv)$ (definition of Pt-score), using contraction rule, we have $\rvy(t)\independent\rt,\rvv|\PS_t$ for both $t$. 
\end{proof}
Apply the proposition to our setting, we have eq.~\eqref{eq:ps_indp}.

Note particularly, the proposition implies $\rvy(t)\independent \rt|\PS_t$ (using decomposition rule). Thus, if $\PS(\rvv)$ is a P-score, then $\PS$ also gives weak ignorability, which is a nice property shared with balancing score, as we will see immediately.

Prognostic scores are closely related to the important concept of balancing score \citep{rosenbaum1983central}. 
\begin{definition}[Balancing score]
\label{bscore}
$\bm\beta(\rvv)$, a function of random variable $\rvv$, is a balancing score if $\rt \independent \rvv|\bm\beta(\rvv)$.
\end{definition}
\begin{proposition}
Let $\bm\beta(\rvv)$ be a function of random variable $\rvv$. $\bm\beta(\rvv)$ is a balancing score if and only if $f(\bm\beta(\rvv))=p(\rt=1|\rvv)\coloneqq e(\rvv)$ for some function $f$ (or more formally, $e(\rvv)$ is $\bm\beta(\rvv)$-measurable). Assume further that $\rvv$ gives weak ignorability, then so does $\bm\beta(\rvv)$.
\end{proposition}
Obviously, the \textit{propensity score} $e(\rvv):=p(\rvt=1|\rvv)$, the propensity of assigning the treatment given $\rvv$, is a balancing score (with $f$ be the identity function). Also, given any invertible function $\vv$, the composition $\vv \circ \bm\beta$ is also a balancing score since $f\circ \vv^{-1}(\vv \circ \bm\beta(\rvv))=f(\bm\beta(\rvv))=e(\rvv)$.

Compare the definition of balancing score and prognostic score, we can say balancing score is sufficient for the treatment $\rt$ ($\rt \independent \rvv|\bm\beta(\rvv)$), while prognostic score (Pt-score) is sufficient for the potential outcomes $\y(t)$ ($\y(t) \independent \rvv|\PS_t(\rvv)$). They complement each other; conditioning on either deconfounds the potential outcomes from treatment, with the former focuses on the treatment side, the latter on the outcomes side.

\subsection{VAE, Conditional VAE, and iVAE}
\label{vaes}
VAEs \citep{kingma2019introduction} are a class of latent variable models with latent variable $\rvz$, and observable $\rvy$ is generated by the decoder $p_\vtheta(\rvy|\rvz)$.
In the standard formulation \citep{DBLP:journals/corr/KingmaW13}, the variational lower bound $\mathcal{L}(\rvy;\vtheta,\bm\phi)$ of the
log-likelihood is derived as:
\begin{equation}
\label{elbo_vae}
\begin{split}
      \log p(\rvy) \geq \log p(\rvy) - \KL(q(\rvz|\rvy)\Vert p(\rvz|\rvy)) \\ 
      = \E_{\vz \sim q}\log p_\vtheta(\rvy|\vz) - \KL(q_{\bm\phi}(\rvz|\rvy)\Vert p(\rvz)),
\end{split}
\end{equation}
where $\KL$ denotes KL divergence and the encoder $q_{\bm\phi}(\rvz|\rvy)$ is introduced to approximate the true posterior $p(\rvz|\rvy)$. The decoder $p_\vtheta$ and encoder $q_{\bm\phi}$ are usually parametrized by NNs. We will omit the parameters $\vtheta,{\bm\phi}$ in notations when appropriate.

The parameters of the VAE can be learned with stochastic gradient variational Bayes. 
With Gaussian latent variables, the KL term of $\mathcal{L}$ has closed form, while the first term can be evaluated by drawing samples from the approximate posterior $q_{\bm\phi}$ using the reparameterization trick \citep{DBLP:journals/corr/KingmaW13}, then, optimizing the evidence lower bound (ELBO) $\E_{\rvy \sim \mathcal{D}}(\mathcal{L}(\vy))$ with data $\mathcal{D}$, we train the VAE efficiently. 

Conditional VAE (CVAE) \citep{sohn2015learning,kingma2014semi} adds a conditioning variable $\rvc$, usually a class label, to standard VAE (See Figure \ref{f:vae}).
With the conditioning variable, CVAE can give better reconstruction of each class. The variational lower bound is
\begin{equation}
    \log p(\rvy|\rvc) \geq \E_{\vz \sim q}\log p(\rvy|\vz,\rvc) - \KL(q(\rvz|\rvy,\rvc)\Vert p(\rvz|\rvc)). 
\end{equation}
The conditioning on $\rvc$ in the prior is usually omitted \citep{doersch2016tutorial}, i.e., the prior becomes $\rvz \sim \mathcal{N}(\bm0, \mI)$ as in standard VAE, since the dependence between $\rvc$ and the latent representation is also modeled in the encoder $q$. Moreover, unconditional prior in fact gives better reconstruction because it encourages learning representation independent of class, similarly to the idea of beta-VAE \citep{higgins2016beta}.

As mentioned, \textit{identifiable} VAE (iVAE) \citep{khemakhem2020variational} provides the first identifiability result for VAE, using auxiliary variable $\x$. It assumes $\rvy \independent \x|\rvz$, that is, $p(\rvy|\rvz,\x)=p(\rvy|\rvz)$. The variational lower bound is
\begin{equation}
\begin{split}
    &\log p(\rvy|\x) \geq \log p(\rvy|\x) - \KL(q(\rvz|\rvy,\x)\Vert p(\rvz|\rvy,\x)) \\
    &\E_{\vz \sim q}\log p_{\vf}(\rvy|\vz) - \KL(q(\rvz|\rvy,\x)\Vert p_{\bm T,\bm\lambda}(\rvz|\x)),
\end{split}
\end{equation}
where $\rvy=\vf(\rvz)+\bm\epsilon$, $\bm\epsilon$ is additive noise, and $\rvz$ has exponential family distribution with sufficient statistics $\bm T$ and parameter $\bm \lambda(\x)$. Note that, unlike CVAE, the decoder does \textit{not} depend on $\x$ due to the independence assumption.

Here, \textit{identifiability of the model} means that the functional \textit{parameters} $(\vf,\bm T,\bm\lambda)$ can be identified (learned) up to certain simple transformation. Further, in the limit of $\bm\epsilon \to \bm0$, iVAE solves the nonlinear ICA problem of recovering  $\rvz=\vf^{-1}(\rvy)$.

\subsection{Comparisons to CEVAE}
\paragraph{Motivation}
Our method is motivated by the sufficient scores. In particular, we introduce B*-scores, which are more applicable than prognostic scores \citep{hansen2008prognostic} and balancing scores \citep{rosenbaum1983central}. And our VAE model is directly based on equations \eqref{eq:cate_by_bts} and \eqref{eq:cate_by_bs} which give CATE from B*-scores. There is no need to recover the hidden confounder in our framework.

CEVAE is motivated by exploiting proxy variables, and its intuition is that the hidden confounder $\rconf$ can be recovered by VAE from proxy variables. 

\paragraph{Applicability}
As a result, proxy variable $\x$ is contained as a special case as shown in our Figure \ref{fig:setting}. 

CEVAE assumes a specific structure among the variables (their Figure 1). In particular, their covariate $\x$, 1) can only contain descendant proxies, 2) cannot affect the outcome directly, and 3), as implicitly assumed in their (2) for decoder, cannot affect the treatment also. That is, their problem setting is just our Figure \ref{fig:setting} with only one possibility $\x=\x_{p_d}$. 

\paragraph{Architecture}
Our model is naturally based on \eqref{eq:cate_by_bts}, particularly the independence properties of Bt-score. And as a consequence, our VAE architecture is a natural combination of iVAE and CVAE (see Figure \ref{f:vae}). Our ELBO \eqref{elbo} is derived by standard variational lower bound. 

On the other hand, the architecture of CEVAE is more ad hoc and complex. Its decoder follows the graphical model of descendant proxy mentioned above, but adds an ad hoc component to mimic TARnet \citep{shalit2017estimating}: it uses separated NNs for the two potential outcomes. We tried similar idea on IHDP dataset, and, as we show in Sec.~5.2, it has basically no merits for our method, because we have a principled balancing as in Sec.~\ref{sec:estimation}.

The encoder of CEVAE is more complex. To have post-treatment estimation, $q(\rt|\x)$ and $q(\y|\x,\rt)$ are added into the encoder. As a result, the ELBO of CEVAE has two additional likelihood terms corresponding to the two distributions. However, in our Intact-VAE, post-treatment estimation is given naturally by our standard encoder, thanks to the correspondence between our model and \eqref{eq:cate_by_bts}. 

\paragraph{Justification}
We give identification under specific and general assumptions in Theorem \ref{id_tion}, and consistent estimation in Corollary \ref{th:estimation}, given the consistency of VAE, which is widely assumed in practice. Moreover, we carefully distinguish assumptions on true generating process and assumptions on our model, and identify the assumptions that are important for causality. 

There are few theoretical justifications for CEVAE. Their Theorem 1 directly assumes the joint distribution $p\st(\x,\y,\rt,\rconf)$ including hidden confounder $\rconf$ is recovered, then identification is trivial by using the standard adjustment equation \eqref{eq:id}. The theorem is in essence no more than giving an example where \eqref{eq:id} works. 

However, as we mentioned in Introduction and Sec.~\ref{sec:setup}, the challenge is exactly that the confounder is hidden, unobserved. Many years of work was done in causal inference, to derive conditions under which hidden confounder can be (partially) recovered \citep{greenland1980effect, kuroki2014measurement, miao2018identifying}. In particular, \cite{miao2018identifying} gives the most recent identification result for proxy setting, which requires very specific two proxies structure, and other completeness assumptions on distributions. Thus, it is unreasonable to believe that VAE, with simple descendant proxies, can recover the hidden confounder. 

Moreover, the identifiability of VAE itself is a challenging problem. As mentioned in Introduction and Sec.~\ref{vaes}, \cite{khemakhem2020variational} is the first identifiability result for VAE, but it only identifies equivalence class, not a unique representation function. Thus, it is also unconvincing that VAE can learn a unique latent distribution, without certain assumptions. 

As we show in Sec.~5.1, for relatively simple synthetic dataset, CEVAE can not robustly recover the hidden confounder, even only up to transformation, while our method can (though, again, this is not needed for our method).

\section{Theoretical exposition}

\subsection{Identifiability of model parameters}
The main part of our model identifiability is essentially the same as that of Theorem 1 in \cite{khemakhem2020variational}, but now adapted to the dependency on $t$. Here we give an outline of the proof, and the details can be easily filled by referring to \cite{khemakhem2020variational}.

\begin{theorem}[Model identifiability]
\label{idmodel}
Given the family $p_{\vtheta}(\rvy,\rvz|\rvx,\rt)$ specified by \eqref{model_indep} and \eqref{model_param}\footnote{We specified factorized Gaussians in \eqref{model_param} and they show good performance in our experiments. But our theorems can be extended to general exponential families, see \cite{khemakhem2020variational}.}, for $\rt=t$, assume  

\renewcommand{\labelenumi}{\roman{enumi})}
\begin{enumerate}
\def\theenumi{\roman{enumi})}

\item $\vf(\rvz)$ \footnote{Here we mean $\vf_t(\vz)$. In \textit{this subsection and related Sec.~\ref{notenough}}, we will refer to quantities when $\rt=t$ is given, and we will omit the subscripts $t$ when appropriate.} is injective and differentiable;

\item $\vg$ is fixed (i.e. $\vg$ is in fact \emph{not} a parameter);

\item \label{ass:inv_jac} there exist $2n+1$ points $\vx_0,...,\vx_{2n}$ such that the $2n$-square matrix $\mL \coloneqq [\bgamma_0,...,\bgamma_{2n}]$ is invertible, where $\bgamma_k\coloneqq\blambda(\vx_k)-\blambda(\vx_0)$.
\end{enumerate}

Then, given $\rt=t$, the family is \emph{identifiable} up to an equivalence class. That is, if $p_{\vtheta}(\rvy|\rvx,\rt=t)=p_{\vtheta'}(\rvy|\rvx,\rt=t)$\footnote{$\vtheta'=(\vf',\vh',\vk')$ is another parameter giving the same distribution. In this paper, symbol $'$ (prime) always indicates another parameter (variable, etc.) in the equivalence class.
}, we have the relation between parameters 
\begin{equation}
\label{eq:class}
     \vf^{-1}(\rvy_t) = \diag(\va)\vf'^{-1}(\rvy_t) + \vb \coloneqq \mathcal{A}(\vf'^{-1}(\rvy_t))
\end{equation}
where $\rvy_t \in \mathcal{Y}_t$, the range of $\vf_t$, and $\diag(\va)$ is an invertible $n$-diagonal matrix and $\vb$ is a $n$-vector, both depend on $\bm\lambda_t$. 
\end{theorem}

The assumptions are all inherited from iVAE. Note that, to have \eqref{eq:class}, we only need the same \textit{observational} distribution $p(\rvy|\rvx,\rt=t)$, but this leaves room for different \textit{latent} distributions. Also, by definition of inverse, we have $\vf' = \vf \circ \mathcal{A}$, and this is the essence of the identifiability. 

Intuitively, if iii) does \textit{not} hold, then the support of $\blambda(\x)$ should be in a $(2n-1)$-dimensional space. Thus, iii) holds easily in practice, if the dimensions of $\blambda(\x)$ are \textit{linearly independent}.

In the proof, all equations and variables should condition on $t$, and we omit the conditioning in notation for convenience. 

\begin{proof}[Proof of Theorem \ref{idmodel}]

Using i) and ii) , we transform $p_{\vf,\bm\lambda}(\rvy|\rvx,t)=p_{\vf',\bm\lambda'}(\rvy|\rvx,t)$ into equality of noiseless distributions, that is, 
\begin{equation}
    q_{\vf',\bm\lambda'}(\rvy)=q_{\vf,\bm\lambda}(\rvy)\coloneqq p_{\bm\lambda}(\vf^{-1}(\rvy)|\rvx,t)vol(\mJ_{\vf^{-1}}(\rvy))\mathbb{I}_{\mathcal{Y}}(\rvy)
\end{equation}
where $p_{\bm\lambda}$ is the Gaussian density function of the conditional prior defined in \eqref{model_param} and $vol(A)=\sqrt{\det A A^T}$.
$q_{\vf',\bm\lambda'}$ is defined similarly to $q_{\vf,\bm\lambda}$.

Then, plug \eqref{model_param} into the above equation, and take derivative on both side at the $\vx^o$ in \ref{ass:inv_jac}, we have
\begin{equation}
    \mathcal{F'}(\rvy)=\mathcal{F}(\rvy)\coloneqq\mL^T\vt(\vf^{-1}(\rvy))-\bbeta
\end{equation}
where $\vt(\rvz)\coloneqq(\rvz,\rvz^2)^T$ is the sufficient statistics of factorized Gaussian, and $\bbeta_t\coloneqq(\alpha_t(\vx_1)-\alpha_t(\vx_0),...,\alpha_t(\vx_{2n})-\alpha_t(\vx_0))^T$ where $\alpha_t(\rvx;\bm\lambda_t)$ is the log-partition function of the conditional prior in \eqref{model_param}. $\mathcal{F'}$ is defined similarly to $\mathcal{F}$, but with $\vf',\bm\lambda',\alpha'$

Since $\mL$ is invertible, we have 
\begin{equation}
    \vt(\vf^{-1}(\rvy))=\mA\vt(\vf'^{-1}(\rvy))+\vc
\end{equation}
where $\mA = \mL^{-T}\mL'^{T}$ and $\vc = \mL^{-T}(\bbeta-\bbeta')$.

The final part of the proof is to show, by following the same reasoning as in Appendix B of \cite{sorrenson2019disentanglement}, that $\mA$ is a sparse matrix such that
\begin{equation}
    \mA=\begin{pmatrix}
    \diag(\va) & \mO \\
    \diag(\vu) & \diag(\va^2)
    \end{pmatrix}
\end{equation}
where $\mA$ is partitioned into four $n$-square matrices. Thus
\begin{equation}
    \vf^{-1}(\rvy)=\diag(\va)\vf'^{-1}(\rvy)+\vb
\end{equation}
where $\vb$ is the first half of $\vc$.
\end{proof}

\subsection{General balancing for Theorem \ref{th:bestimation}}
\label{bcovar}

\begin{theorem}[Theorem \ref{th:bestimation}, generalized]
In Theorem \ref{th:bestimation}, replace \ref{ass:bcovar} with the following, we have the same result.
\renewcommand{\labelenumi}{\roman{enumi})}
\begin{enumerate}
\def\theenumi{\roman{enumi})}
\item \label{ass:g_bcovar} \emph{(General balancing).} There exist $2n+1$ points $\vx_0,...,\vx_{2n}$, \emph{invertible} $2n$-square matrix $\mC$, and $2n$-vector $\vd$, such that $\mL_0\inv\mL_1=\mC$ and $\bbeta_0-\mC^{-T}\bbeta_1=\vd$, where $\mL$ as defined by iii) in Theorem \ref{idmodel}, $\bbeta_t\coloneqq(\alpha_t(\vx_1)-\alpha_t(\vx_0),...,\alpha_t(\vx_{2n})-\alpha_t(\vx_0))^T$, and $\alpha_t(\rvx;\bm\lambda_t)$ is the log-partition function of the conditional prior in \eqref{model_param};
\end{enumerate}
\end{theorem}
Note, $\bbeta_t$ was seen once in the proof of Theorem \ref{idmodel}. Assumption \ref{ass:g_bcovar} is a \textit{necessary and sufficient} condition for $\mathcal{A}_0=\mathcal{A}_1$. Thus, we can see it as a general balancing, because the balanced estimator \eqref{bestimator} is enabled by it.

\begin{proof}[Proof of generalized Theorem \ref{th:bestimation}] 
\ref{ass:g_bcovar} implies \ref{ass:inv_jac} in Theorem \ref{idmodel}, since ${\mC}$ is invertible. So, we have all assumptions for Theorem \ref{idmodel}. 

We only need to further prove $\mathcal{A}_0=\mathcal{A}_1$, and the rest is the same as Theorem \ref{th:bestimation}. Also, it is apparent from below that \ref{ass:g_bcovar} is a \textit{necessary and sufficient} condition for $\mathcal{A}_0=\mathcal{A}_1$, if other assumptions of Theorem \ref{idmodel} are given.



We repeat the core quantities from Theorem \ref{idmodel} here: $\mA_t = \mL_t^{-T}\mL_t'^{T}$ and $\vc_t = \mL_t^{-T}(\bbeta_t-\bbeta'_t)$.

From \ref{ass:g_bcovar}, we immediately have
\begin{equation}
    \mL_0^{-1}\mL_1=\mL_0'^{-1}\mL'_1={\mC} \iff \mA_0=\mA_1
\end{equation}

Also from \ref{ass:bcovar},
\begin{equation}
    \begin{split}
        \mL_0\inv\mL_1&=\mC \iff \\
        \mL_0^{-T}\mC^{-T}&=\mL_1^{-T} \\
        \bbeta_0-\mC^{-T}\bbeta_1&=\bbeta'_0-\mC^{-T}\bbeta'_1=\vd \iff \\
        \mC^{T}(\bbeta_0-\bbeta'_0)&=\bbeta_1-\bbeta'_1
    \end{split}
\end{equation}
Multiply line 2 and 4 of (\theequation), we have $\vc_0=\vc_1$. Now we have $\mathcal{A}_0=\mathcal{A}_1\coloneqq\mathcal{A}$.
\end{proof}

Assumption \ref{ass:g_bcovar} is general despite (or thanks to) the involved formulation. Let us see its generality even under a highly special case: 
$\mC=c\mI$ and $\vd=\bm0$. Then, $\mL_0\inv\mL_1=c\mI$ requires that, $\vh_1(\vx_k)-c\vh_0(\vx_k)$ is the same for $2n+1$ points $\vx_k$. This is easily satisfied except for $n \gg m$ where $m$ is the dimension of $\x$, which \textit{rarely} happens in practice. And, $\bbeta_0-\mC^{-T}\bbeta_1=\vd$ becomes just $\bbeta_1=c\bbeta_0$. This is equivalent to $\alpha_1(\vx_k)-c\alpha_0(\vx_k)$ same for $2n+1$ points, again fine in practice. 

However, the high generality comes with price. First, the general balancing assumption only ensures the posterior is balanced, while the conditional prior may not. Second, verifying \ref{ass:g_bcovar} using data is challenging, particularly with high-dimensional covariate and latent variable. Although we believe fast algorithms for this purpose could be developed, the effort would be nontrivial. 

The applicability under the special case, together with the two possible limitations, motivate us to use the extreme case in \ref{ass:bcovar} of Theorem \ref{th:bestimation}, which corresponds to $\mC=\mI$ and $\vd=\bm0$. Given the above analysis, we are confident that the extreme case usually works better in practice, and our experiments support this.



\subsection{Consistency of VAE}
The following is a refined version of Theorem 4 in \cite{khemakhem2020variational}.
The result is proved by assuming: i) our VAE is flexible enough to ensure the ELBO is tight (equals to the true log likelihood) for some parameters; ii) the optimization algorithm can achieve the \textit{global} maximum of ELBO (again equals to the log likelihood).
\begin{proposition}[Consistency of VAE]
\label{consistency}
Given the VAE model \eqref{model_indep}--\eqref{eq:enc}, and let $p^*(\rvx,\rvy,\rt)$ be the true observational distribution, assume 

\renewcommand{\labelenumi}{\roman{enumi})}
\begin{enumerate}
\def\theenumi{\roman{enumi})}
    \item there exists $(\bar{\vtheta}, \bar{\bm\phi})$ such that $p_{\bar{\vtheta}}(\rvy|\rvx,\rt)=p^*(\rvy|\rvx,\rt)$ and $p_{\bar{\vtheta}}(\rvz|\rvx,\rvy,\rt)=q_{\bar{\bm\phi}}(\rvz|\rvx,\rvy,\rt)$;
    
    \item \label{ass:prime} the ELBO $\E_{\mathcal{D} \sim p^*}(\mathcal{L}(\rvx,\rvy,\rt; \vtheta, \bm\phi))$ \eqref{elbo} can be optimized to its global maximum at $(\vtheta', \bm\phi')$;
\end{enumerate}

Then, in the limit of infinite data, $p_{\vtheta'}(\rvy|\rvx,\rt)=p^*(\rvy|\rvx,\rt)$ and $p_{\vtheta'}(\rvz|\rvx,\rvy,\rt)=q_{\bm\phi'}(\rvz|\rvx,\rvy,\rt)$.
\end{proposition}

\begin{proof}[Proof]
From i), we have $\mathcal{L}(\rvx,\rvy,\rt; \bar{\vtheta}, \bar{\bm\phi})=\log p^*(\rvy|\rvx,\rt)$. But we know $\mathcal{L}$ is upper-bounded by $\log p^*(\rvy|\rvx,\rt)$. So, $\E_{\mathcal{D} \sim p^*}(\log p^*(\rvy|\rvx,\rt))$ should be the global maximum of the ELBO (even if the data is finite).

Moreover, note that, for any $(\vtheta, \bm\phi)$, we have $\KL(p_{\vtheta}(\rvz|\rvx,\rvy,\rt) \Vert q_{\bm\phi}(\rvz|\rvx,\rvy,\rt) \geq 0$ and, in the limit of infinite data, $\E_{\mathcal{D} \sim p^*}(\log p_{\vtheta}(\rvy|\rvx,\rt)) \leq \E_{\mathcal{D} \sim p^*}(\log p^*(\rvy|\rvx,\rt))$. Thus, the global maximum of ELBO is achieved \textit{only} when $p_{\vtheta}(\rvy|\rvx,\rt)=p^*(\rvy|\rvx,\rt)$ and $p_{\vtheta}(\rvz|\rvx,\rvy,\rt)=q_{\bm\phi}(\rvz|\rvx,\rvy,\rt)$.
\end{proof}


\subsection{Details and formal results for example generating processes}
The propositions and proofs here give the intuition for development of Lemma 1 and Theorem \ref{id_tion}.

\begin{proposition}[Identification for Example 1]
\label{id_tion1}
Given the family $p_{{\vf},{\bm\lambda}}(\rvy,\rvz|\rvx,\rt)$ specified by \eqref{model_indep} and \eqref{model_param}, 
and true data distribution $p^*(\rvx, \rvy, \rt,\rconf, \rve)$ generated as in Example 1,
assume, in our model,

 


\renewcommand{\labelenumi}{\roman{enumi})}
\begin{enumerate}
\def\theenumi{\roman{enumi})}
    

    \item $\vf_t$ is injective;
    \item $\vg_t(\rvz) = \bm0$;
    \item $\rvz$ is \emph{not} lower-dimensional than $\PS^*$;

    
    
    
    
    
    
\end{enumerate}

Then, 
$\bar{\z}\coloneqq{\bar{\vf}}^{-1}_{\rt}(\rvy)$ is a Pt-score (and also Bt-score) of $(\rvx,\rconf,\rve_b)$, where 
 $(\rvx, \rvy,\rt, \rconf,\rve) \sim p^*$.


\end{proposition}

\begin{proof}[Proof of Proposition \ref{id_tion1}]

We have $\vf^*_{\rt}(\PS^*)=\y$ in the true generating process. From ii), we have $\vf_{\rt}(\rvz)=\rvy$ in our decoder. 
Set $\vf_{\rt}=\bar{\vf}_{\rt}$ (any injective function)
in our decoder, 
then, for any $\bm\lambda$, we have degenerate posterior $p_{\bar{\vf},\bm\lambda}(\z|\rvy,\rvx,\rt)=\delta(\rvz-\bar{\vf}^{-1}_{\rt}(\rvy))$. 



For any injective $\bar{\vf}_t$, since $\vf^*_t$ is also injective, and further from iii), we can define $\vd_t\coloneqq\bar{\vf}_t^{-1}\circ\vf^*_t$, and $\vd_t$ is injective. So we have $\bar{\vf}_{\rt}(\rvz)=\vf^*_{\rt}(\vd_{\rt}^{-1}(\rvz))=\rvy$, where $(\rvx,\rvy,\rt) \sim p^*$. Due to the injectivity of $\vf^*_t$, this can happen only when $\vd_{\rt}^{-1}(\rvz)=\PS^*$ for all $\rvx$. Due to the injectivity of $\vd_t$, for any $\bar{\vf}$, $\rvz=\vd_{\rt}(\PS^*)$ is a Pt-score. 
\end{proof}

The essence of Proposition \ref{id_tion1} can be captured in one equation, which says the counterfactual prediction, given by $\bar{\vf}$ and the respective Pt-score $\bar{\z}$, is the same as truth:
\begin{equation}
\label{essence}
\begin{split}
    \bar{y}_{1-t}\coloneqq \bar{\vf}_{1-t}(\bar{\z}_{1-t})
    &=\bar{\vf}_{1-t}\circ\vd_{1-t}(\PS\st) \\
    &=\vf\st_{1-t}({\PS\st})=y({1-t}).
\end{split}
\end{equation}


Interestingly, this also implies we have the identification of \textit{individual} treatment effects (defined by $\rvy(1)-\rvy(0)$). Also, in the true generating process, $\y$ itself is in fact a Pt-score, since it is a function of $\PS\st$. Both are not possible if we have non-zero outcome noise $\rve_a$.



Next, we examine Example 2. We first give a simple, linear outcome special case where $\BS\st$ is a B-score of $\x$. Other cases prevail.

For Example 2, $\E(\y(t)|\BS\st(\x),\x)=\E(\vf^*_t(\BS^*, \rconf,\rve_b)|\rvx)+\E(\rve_a|\x)$. If, in particular,  $\vf^*_t$ is \textit{linear}, then the first term becomes $\vf^*_t(\BS^*, \E(\rconf|\rvx),\E(\rve_b|\rvx))$. Now, if $\E(\rconf|\rvx)=\E(\rconf|\BS^*,\rt=t)$, $\E(\rve_b|\rvx)=\E(\rve_b|\BS^*,t)$, and $\E(\rve_a|\x)=\E(\rve_a|\BS\st,t)$ (or simply, $\rconf,\rve$ satisfy mean exchangeability \citep{dahabreh2019generalizing} given $\BS\st$), then $\BS^*$ is a B-score with $\F_t(\BS^*) = \vf^*_t(\BS^*, \E(\rvz|\BS^*,t),\E(\rve_b|\BS^*,t))+\E(\rve_a|\BS\st,t)$.


In the following proposition, we assume $\rve_a$ to be Gaussian, to enable the ``noise matching'' condition we use many times in this paper. Note, however, as we mentioned in the main text, this is not a real restriction because the outcome noise in our model can be readily extended to non-Gaussian. 

\begin{proposition}[Identification for Example 2]
\label{id_tion2}
Given the family $p_{{\vf},{\bm\lambda}}(\rvy,\rvz|\rvx,\rt)$ specified by \eqref{model_indep} and \eqref{model_param}, 
and true data distribution $p^*(\rvx, \rvy, \rt,\rconf, \rve)$ generated as in Example 2,
assume 

\renewcommand{\labelenumi}{\roman{enumi})}
\begin{enumerate}
\def\theenumi{\roman{enumi})}
    \item $\F_t$ is injective;
    \item $\rve_a$ is factorized Gaussian with zero-mean and $\bm\sigma_{a,\rt}$;

\hspace{-8mm} And in our model,

    \item $\vf_t$ is injective;
    \item $\vg_t(\rvz)= \bm\sigma_{a,t}^2$ and $\vk_t(\rvx) = \bm0$;
    \item $\rvz$ is \emph{not} lower-dimensional than $\BS^*$;

\end{enumerate}

Then, $p_{\bar{\vf},\bar{\bm\lambda}}(\rvy|\rvx,\rt)=p^*(\rvy|\rvx,\rt)$ implies $\bar{\z}\coloneqq\bar{\vh}_{\rt}(\rvx)$ is a Bt-score (but not Pt-score) of $\rvx$
where $(\rvx, \rt) \sim p^*$.

\end{proposition}
We set $\vk_t(\rvx) = \bm0$ in the model because we know the B-score is a function of $\x$. Then, our conditional prior is degenerate like the posterior in Example 1, and this is needed to have $\bar{\vh}_{\rt}(\rvx)$ as a Bt-score. But, without the degeneration, the prior still gives Bt-score, as seen in Theorem \ref{id_tion}.

Also, it seems, from the above example, that the non-zero mean of $\rve_a$ as a function of $\BS\st$ can be allowed. This, however, is in essence the same as current formulation, because we can equivalently subtract the mean from noise, and add it to $\vf\st_t$ term. As we can see in Lemma 1 and its proof, the real requirement is that the final $\F_t$ corresponding to the B-score is injective.

\begin{proof}[Proof of Proposition \ref{id_tion2}]
Note, unlike Proposition \ref{id_tion1}, we require $p_{\bar{\vf},\bar{\bm\lambda}}(\rvy|\rvx,\rt)=p^*(\rvy|\rvx,\rt)$. This, together with ii) and $\vg_t(\rvz)= \bm\sigma_{a,t}^2$, implies that the distributions of means of $\bar{\y}$ and $\y\st$ should be the same given $\x,\rt$. That is, in the notion introduced in Lemma 1, $\F_{\rt}(\BS\st)\deq\bar{\vf}_{\rt}(\z)\xt$.

Also, note the conditional prior is degenerate: $p_{\bm\lambda}(\z|\rvx,\rt)=\delta(\rvz-\bar{\vh}_{\rt}(\rvx))$.

Finally, using the co-injectivity of $\F_t,\vf_t$, we have $\BS\st\deq\F_t\inv\circ\bar{\vf}_t(\bar{\vh}_{\rt}(\rvx))\xt$, and, importantly, $\F_t\inv\circ\bar{\vf}_t$ is also injective. Note that, $\BS\st$ is a function of $\x$, so the functions $\BS\st$ and $\bar{\vh}_t$ are different only up to an injective mapping. Thus, $\bar{\vh}_{\rt}(\rvx)$ should be a Bt-score.
\end{proof}

\subsection{Identifiability of representation (is not enough)}
\label{notenough}

Nowhere in the main text refers this and the next subsection, so you can omit them if not interested. However, if reading, you may gain insight of how we came to our final theoretical formulation.

Here we explain that the model identifiability given in Theorem \ref{idmodel} alone is, albeit interesting, not enough for estimation of treatment effects. 

The importance of model identifiability can be seen clearly in the following corollary. That is, given $\rt=t$, the latent representation can be identified up to an invertible element-wise affine transformation. It can be easily understood by noting that, with the small noise and the injective $\vf$, the decoder degenerates to deterministic function and the latent representation $\rvz=\vf^{-1}(\rvy)$.

\begin{corollary}
\label{idrepre}
In Theorem \ref{idmodel}, let $\bm\sigma_{\rvy,t} = \bm0$, then $\rvz = \mathcal{A}_t(\rvz')$. 
\end{corollary}

The good news is that, all the possible latent representations in our model are equivalent if we consider their independence relationships with any random variables, because any two of them are related by an \textit{invertible} mapping. However, the bad news is that, this holds only given $\rt=t$, while the definition of B/P-score involves both $t$. 

Consider how the \textit{recovered} $\rvz'$ would be used. For a control group ($t=0$) data point $(\vx, y, 0)$, the real challenge under finite sample is to predict the counterfactual outcome $y(1)$. Taking the observation, the encoder will output a posterior sample point $\vz'_0=\vf_0'^{-1}(y)=\mathcal{A}_0^{-1}(\vz_0)$ (with zero outcome noise, the encoder degenerates to a delta function: $q(\rvz|\vx, y, 0)=\delta(\rvz-\vf_0'^{-1}(y))$). Then, we should do \textit{counterfactual inference}, using decoder with counterfactual assignment $t=1$: $y_1'=\vf'_1(\vz'_0)=\vf_1\circ\mathcal{A}_1(\mathcal{A}_0^{-1}(\vz_0))$. This prediction can be arbitrary far from the truth $y(1)=\vf_1(\vz_0)$, due to the difference between $\mathcal{A}_1$ and $\mathcal{A}_0$. More concretely, this is because when learning the decoder, only the posterior sample of the treatment group ($t=1$) is fed to $\vf'_1$, and the posterior sample is different to the true value by the affine transformation $\mathcal{A}_1$, while it is $\mathcal{A}_0$ for $\vz_0'$.

Now we know what we need: $\mathcal{A}_0=\mathcal{A}_1$ so that the equivalence of independence holds unconditionally; and, there exists at least one representation that is indeed a B-score. Then, \textit{any} representation in our model will be a B-score. These indeed are what we have in Sec.~\ref{sec:estimation}.

\begin{proof}[Proof of Corollary 1]
In this proof, all equations and variables should condition on $t$, and we omit the conditioning in notation for convenience. 

When $\bm\sigma_{\rvy} = \bm0$, the decoder degenerates to a delta function: $p(\rvy|\rvz)=\delta(\rvy-\vf(\rvz))$, we have $\rvy=\vf(\rvz)$ and $\rvy'=\vf'(\rvz')$. 
For any $\vy$ in the common support of $\rvy,\rvy'$, there exist a \textit{unique} $\vz$ and a \textit{unique} $\vz'$ satisfy $\vy=\vf(\vz)=\vf'(\vz')$ (use injectivity). Substitute $\vy=\vf(\vz)$ into the l.h.s of \eqref{eq:class}, and $\vy=\vf'(\vz')$ into the r.h.s, 
so we get $\rvz=\mathcal{A}(\rvz')$. The result for $\vf$ follows. 
\end{proof}
A technical detail is that, $\vz, \vz'$ might not always be related by $\mathcal{A}$, because we used the \textit{common} support of $\rvy,\rvy'$ in the proof. Thus, the relation holds for partial supports of $\rvz, \rvz'$ correspond to the common support of $\rvy,\rvy'$. This problem disappears if we have the a consistent learning method (see Proposition \ref{consistency}).

\subsection{Balancing covariate and its two special cases}
Here we demonstrate part of our old, limited, theoretical formulation, and extract some insights from it.

The following definition was used in the old theory.
The importance of this definition is immediate from the definition of balancing score, that is, if a balancing \textit{covariate} is also a function of $\rvv$, then it is a balancing \textit{score}.

\begin{definition}[Balancing covariate]
Random variable $\rvx$ is a {\em balancing covariate} of random variable $\rvv$ if $\rt \independent \rvv|\rvx$. We also simply say $\rvx$ is \textit{balancing} (or \textit{non}-balancing if it does not satisfy this definition).
\end{definition}

Given that a balancing score of the true (hidden or not) confounder is sufficient for weak ignorability, a natural and interesting question is that, does a balancing covariate of the true confounder also satisfies weak ignorability? The answer is \textit{no}. To see why,
we give the next Proposition indicating that a balancing covariate of the true confounder might \textit{not} satisfy \textit{exchangeability}. 


\begin{proposition}
Let $\rvx$ be a balancing covariate of $\rvv$. If $\rvv$ satisfies exchangeability \emph{and} $\rvy(t) \independent \rvx|\rvv,\rt$, then so does $\rvx$.
\end{proposition}
The proof will use the properties of conditional independence (Proposition \ref{indep_prop}).
\begin{proof}[Proof]
Let $\rvw\coloneqq \rvy(t)$ for convenience. We first write our assumptions in conditional independence, as \textit{A1}. $ \rt \independent \rvv|\rvx$ (balancing covariate), \textit{A2}. $\rvw \independent \rt| \rvv$ (exchangeability given $\rvv$), and \textit{A3}. $\rvw \independent \rvx|\rvv,\rt$.

Now, from \textit{A2} and \textit{A3}, using contraction, we have $\rvw \independent \rvx,\rt|\rvv$, then using weak union, we have $\rvw \independent \rt|\rvx,\rvv$. From this last independence and \textit{A1}, using contraction, we have $\rt \independent \rvv,\rvw|\rvx$. Then $\rt \independent \rvw|\rvx$ follows by decomposition.
\end{proof}

Given this proposition, we know assumptions
\begin{equation}
\label{old_ass}
    \begin{split}
        \text{i) }& \rvy(t) \independent \rt|\rvv \text{ (exchangeability given $\rvv$), } \\
        \text{ii) }& \rt \independent \rvv|\rvx \text{ ($\x$ is a balancing covariate of $\rvv$), and} \\
        \text{iii) }& \rvy \independent \rvx|\rvv,\rt
    \end{split}
\end{equation}
do not imply exchangeability given $\x$, thus seem to be reasonable. 
Note the independence $\rvy(t) \independent \rvx|\rvv,\rt$ assumed in the above proposition implies, but is not implied by, $\rvy \independent \rvx|\rvv,\rt$.
This is because, in general, $\rvy(0) \independent \rvx|\rvv,\rt=1$ and $\rvy(1) \independent \rvx|\rvv,\rt=0$ do not hold.

The assumptions in \eqref{old_ass} were assumed by our old theory, with $\rvv$ is hidden confounder $\rconf$ plus observed confounder $\x_c$. And also note that, iii) is the independence shared by Bt-score.



We examine two important special cases of balancing covariate, which provide further evidence that balancing covariate does not make the problem trivial.

\begin{definition}[Noiseless proxy]
Random variable $\rvx$ is a noiseless proxy of random variable $\rvv$ if $\rvv$ is a function of $\rvx$ ($\rvv=\bm\omega(\rvx)$). 
\end{definition}
Noiseless proxy is a special case of balancing covariate because if $\rvx=\vx$ is given, we know $\vv=\bm\omega(\vx)$ and $\bm\omega$ is a deterministic function, then $p(\rvv|\rvx=\vx)=p(\rvv|\rvx=\vx,\rt)=\delta(\rvv-\bm\omega(\vx))$.
Also note that, a noiseless proxy always has higher dimensionality than $\rvv$, or at least the same. 

Intuitively, if the value of $\rvx$ is given, there is no further uncertainty about $\vv$, so the observation of $\vx$ may work equally well to adjust for confounding. But, as we will see soon, a noiseless proxy of the true confounder does \textit{not} satisfy positivity. 

\begin{definition} [Injective proxy]
Random variable $\rvx$ is an injective proxy of random variable $\rvv$ if $\rvx$ is an injective function of $\rvv$ ($\rvx=\bm\chi(\rvv)$, $\bm\chi$ is injective). 
\end{definition}
Injective proxy is again a special case of noiseless proxy, since, by injectivity, $\rvv=\bm\chi^{-1}(\rvx)$, i.e. $\rvv$ is also a function of $\rvx$. 

Under this very special case, that is, if $\rvx$ is an injective proxy of the true confounder $\rvv$, we finally have $\rvx$ is a balancing score and satisfies weak ignorability, since $\rvx$ is a balancing covariate and a function of $\rvv$. To see this in another way, let $f=e \circ \bm\chi^{-1}$ and $\bm\beta=\bm\chi$ in Proposition \ref{bscore}, then $f(\rvx)=f(\bm\beta(\rvv))=e(\rvv)$. By weak ignorability of $\rvx$, \eqref{eq:id} has a simpler counterpart $\mu_t(\vx) = \E(\rvy(t)|\rvx=\vx) = \E(\rvy|\rvx=\vx,\rt=t)$. Thus, a naive regression of $\rvy$ on $(\rvx, \rt)$ will give a valid estimator of CATE and ATE.  

However, a noiseless but \textit{non}-injective proxy is \textit{not} a balancing score, in particular, positivity might \textit{not} hold. Here, a naive regression will not do. This is exactly because $\bm\omega$ is non-injective, hence multiple values of $\rvx$ that cause non-overlapped supports of $p(\rt=t|\rvx=\vx),t=0,1$ might be mapped to the same value of $\rvv$. An extreme example would be $\rt=\mathbb{I}(\rx>0),\rz=|\rx|$. We can see $p(\rt=t|\rx)$ are totally non-overlapped, but $\forall t,z \neq 0: p(\rt=t|\rz=z)=1/2$.

So far, so good. In the end, what is the problem of balancing covariate? Here it is. \textit{If} the we have the positivity of $\x$ ($p(\rt|\x)>0$ always), then, using the positivity and balancing to get $p(\conf|\vx)=p(\conf|\vx,\rt=t)$ for all $\vx$, we follow \eqref{eq:id},
\begin{equation}
\begin{split}
    \mu_t(\vx)&=\textstyle \int (\int p(y|\conf,\vx,t)ydy)p(\conf|\vx)d\conf \\
    &=\textstyle \int (\int p(y|\conf,\vx,t)ydy)p(\conf|\vx,\rt=t)d\conf \\
    &=\textstyle \int (\int p(y,\conf|\vx,t)d\conf)ydy=\E(\y|\vx,t).
\end{split}
\end{equation}
Naive estimator just works! Thus, \textit{if} $\x$ indeed was a balancing covariate of true confounder, we gave a better method than naive estimator only in the sense that it works without positivity of $\x$. It seems what our old theory really addressed was lack of positivity, another important issue in causal inference \citep{d2020overlap}, but not confounding.

This limited formulation, together with the great experimental performance of our method, motivated us to develop a much more general theory, that is, the theory based on B*-scores in the main text.

There are several lessons learned from the old formulation. First, there may exist cases that exchangeability given $\x$ fails to hold even when positivity of $\x$ holds, but the naive estimator still works. This is related to the fact that the conditional independence based on which balancing score/covariate are defined is not necessary for identification. And we should be able to find weaker but still sufficient conditions for identification, and Bt-score is an example. Second, balancing covariate assumption in \eqref{old_ass} is strong, though may not make a trivial problem. It basically means $\x$, only one of the observables, is sufficient information for treatment assignment. This inspires us to consider both $\x,\y$ in our theory, as in the Bt-score given by our posterior and encoder.



\section{Details and additional results for experiments}

\subsection{Synthetic data}

\begin{figure*}[h]
    \centering
    \includegraphics[width=.9\textwidth]{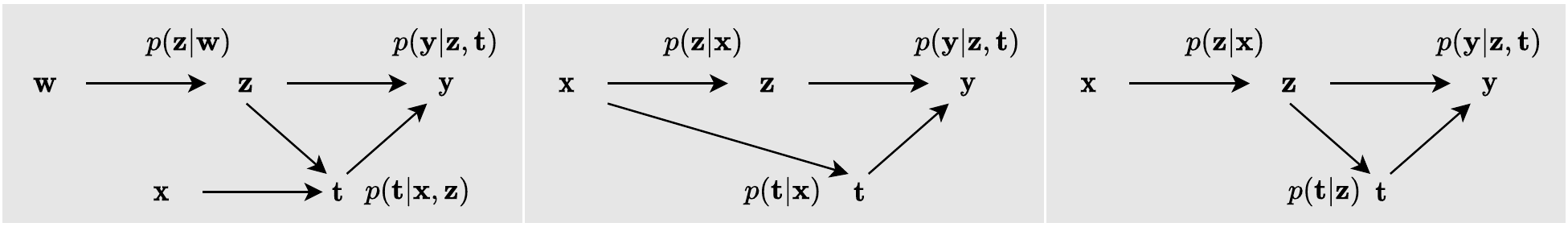}
    \caption{\footnotesize{Graphical models for generating synthetic datasets. From left: IV $\rvx$, ignorability given $\rvx$, and proxy $\rvx$. Note that in the latter two cases, reversing the arrow between $\rvx, \rvz$ does not change any independence relationships, and causal interpretations of the graphs remain the same.}}
    \label{fig:art}
\end{figure*}

We generate data following \eqref{art_model} with $\rz$, $\rvy$ 1-dimensional and $\rvx$ 3-dimensional. $\mu_i$ and $\sigma_i$ are randomly generated in range $(-0.2, 0.2)$ and $(0, 0.2)$, respectively. 

We generate two kinds of outcome models, depending on the type of $f$: linear and nonlinear outcome models use random linear functions and NNs with invertible activations and random weights, respectively. 

We experiment on \textit{three different causal settings} by variations on \eqref{art_model} as following. Also see Figure \ref{fig:art} for graphical models of these three cases.
Instead of taking inputs $\rvx, \rz$ in $l$, we consider two special cases: $l\coloneqq l(\rvx)$, then $\rvx$ fully adjusts for confounding, we are in fact \textit{unconfounded}; and $l\coloneqq l(\rz)$, then we have unobserved confounder $\rz$ and \textit{proxy} $\rvx$ of $\rz$. To introduce $\rvx$ as \textit{instrumental variable}, we generate another 1-dimensional random source $\rw$ in the same way as $\rvx$, and use $\rw$ instead of $\rvx$ to generate $\rz|\rw \sim \mathcal{N}(h(\rw), \beta k(\rw))$. Except indicated above, other aspects of the models are specified by \eqref{art_model}. 

\begin{wrapfigure}{r}{0.25\textwidth}

\vspace{-.4in}
  \begin{center}
    \includegraphics[width=0.25\textwidth]{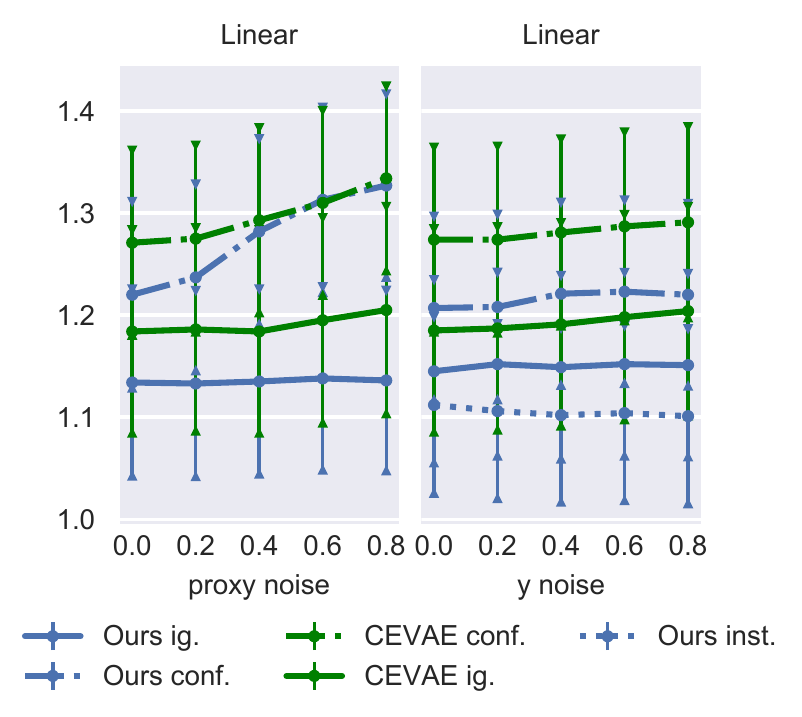}
  \end{center}
  \vspace{-.2in}
  
  \caption{\footnotesize{$\sqrt{\epsilon_i}$ on linear synthetic dataset. Error bar on 100 random models. We adjust one of $\alpha,\beta$ at a time. Results for ATE and post-treatment are similar.}}
\vspace{-.1in}
\label{lin_art}
\end{wrapfigure}

We adjust the outcome and proxy noise level by $\alpha,\beta$ respectively.
The output of $f$ is normalized by $C\coloneqq\Var_{\{\mathcal{D}|\rt=t\}}(f(\rz))$. This means we need to use $0 \leq \alpha < 1$ to have a reasonable level of noise on $\rvy$ (the scales of mean and variance are comparable). Similar reasoning applies to $\rz|\rvx$; outputs of $h,k$ have approximately the same range of values since the functions' coefficients are generated by the same weight initializer.

As shown in Figure \ref{lin_art}, our method again constantly outperforms CEVAE under linear outcome models. Interestingly, linear outcome models seem harder for both methods\footnote{Note that, after generating the outcomes and before the data is used, we normalize the distribution of ATE of the 100 generating models, so the errors on linear and nonlinear settings are basically comparable.}. While we did not dig into this point because this would be a digress from our purpose, we give two possible reasons: 1)  under similar noise levels, the observed outcome values under nonlinear outcome models might be more informative about the values of $\rz$, because nonlinear models are often steeper than linear models for many values of $\rz$; 2) the two true linear outcome models for $t=0,1$ are more similar, particularly when the two potential outcomes are in small and similar ranges, and it is harder to distinguish and learning the two outcome models.

You can find more plots for latent recovery at the end of the paper.

\subsection{IHDP} 
IHDP is based on an RCT where each data point represents a child with 25 features about their birth and mothers. \texttt{Race} is introduced as a confounder by artificially removing all treated children with nonwhite mothers. There are 747 subjects left in the dataset. The outcome is synthesized by taking the covariates (features excluding \texttt{Race}) as input, hence \textit{unconfoundedness} holds given the covariates. Following previous work, we split the dataset by 63:27:10 for training, validation, and testing.

The generating process is as following \citep[Sec.~4.1]{hill2011bayesian}.
\begin{equation}
    \ry(0) \sim \mathcal{N}(e^{\va^T(\x+\vb)},1),\quad \ry(1) \sim \mathcal{N}(\va^T\x-o,1),
\end{equation}
where $\vb$ is a constant bias with all elements equal to $0.5$, $\va$ is a random coefficient, and $o$ is a random parameter adjusting degree of overlapping between the treatment groups.

As we can see, this is an invertible noise model with $\va^T\x$ as the true (necessary) B-score.


\subsection{Pokec} 
To extract information from the network structure, we use Graph Convolutional Network (GCN) \citep{DBLP:conf/iclr/KipfW17} in conditional prior and encoder of Intact-VAE. A difficulty is that, the network $\mG$ and covariates $\mX$ of \textit{all} users are always needed by GCN, regardless of whether it is in training, validation, or testing phase. However, the separation can still make sense if we take care that the treatment and outcome are used only in the respective phase, e.g., $(y_m,t_m)$ of a testing user $m$ is only used in testing.

GCN takes the network matrix $\mG$ and the \textit{whole} covariates matrix $\mX \coloneqq (\vx_1^T,\dotsc,\vx_M^T)^T$, where $M$ is user number, and outputs a representation matrix $\mR$, again for all users. During training, we \textit{select} the rows in $\mR$ that correspond to users in training set. Then, treat this \textit{training representation matrix} as if it is the covariates matrix for a non-networked dataset, that is, the downstream networks in conditional prior and encoder are the same as in the other two experiments, but take $(\mR_{m,:})^T$ where $\vx_m$ was expected as input. And we have respective selection operations for validation and testing. We can still train Intact-VAE including GCN by Adam, simply setting the gradients of non-seleted rows of $\mR$ to 0. 

Note that GCN cannot be trained using mini-batch, instead, we perform batch gradient decent using full dataset for each iteration, with initial learning rate $10^{-2}$. We use dropout \citep{srivastava2014dropout} with rate 0.1 to prevent overfitting.

The pre-treatment $\sqrt{\epsilon_i}$ for \texttt{Age}, \texttt{District}, and \texttt{Join date} confounders are 1.085, 0.686, and 0.699 respectively, practically the same as the ATE \textit{errors}.


\subsection{Additional plots on synthetic datasets}
See last pages.

\section{Discussions}

As we see in the balanced estimator \eqref{bestimator}, our representation $\z$ is in fact capable of \textit{counterfactual inference}: $\hat{t}$ can be different to factual $\rt=t$. Experiments on counterfactual generation, like those in \citet[CausalGAN]{kocaoglu2017causalgan} and \citet[CausalVAE]{yang2020causalvae}, are on the way. 

Since our method works without the recovery of either hidden confounder or true score distribution, we often cannot see apparent relationships between recovered latent representation and the true hidden confounder/scores. It would be nice to directly see the learned representation preserves causal properties, for example, by some causally-specialized metrics, e.g. \cite{suter2019robustly}. 

Despite the formal requirement in Theorem \ref{idmodel} of fixed distribution of noise on $\rvy$, inherited from \cite{khemakhem2020variational}, the experiments show evidence that our method can learn the outcome noise. We observed that, in most cases, allowing the noise distribution to be learned depending on $\rvz,\rt$ improves performance.  Theoretical analysis of this phenomenon is an interesting direction for future work. 

We conjecture that, it is possible to extend model identifiability to conditional noise models $\vg_t(\z)$.
And we expect that the noise on $\y$ can also be identified up to some eq.~class (or joint eq.~class together with $\vf$). In that case, the model identifiability may also be sufficient for causal inference, under some respective assumptions on true generating process (note the current \ref{ass:gen} in Theorem \ref{id_tion} also seems stronger than needed), and our current \ref{ass:gen_n} and \ref{ass:model_latent} in Theorem \ref{id_tion} can be \textit{relaxed} to large extent. Similarly to current $\vf$, we may have identification for a general class of noises. 

Also, our causal theory does not in principle require continuous latent distributions, though in Theorem \ref{th:bestimation} for balanced estimation, differentiability of $\vf$ is inherited from iVAE. Given the fact that currently all nonlinear ICA based identifiability requires differentiable mapping between the latent and observables, directly based on it, theoretical extensions to \textit{discrete} latent variable would be challenging. However, as we see in Sec.~\ref{sec:estimation} and \ref{bcovar}, what is essential for CATE identification is the \textit{same} transformation between true and recovered score distribution for both $t$, but the transformation needs \textit{not} to be affine, and, possibly, neither injective. This opens directions for future extensions, based not necessarily on nonlinear ICA.


\clearpage
\onecolumn
\begin{figure}[h] 
    \vspace*{-0.1in}
    \centering
    \includegraphics[width=.9\textwidth]{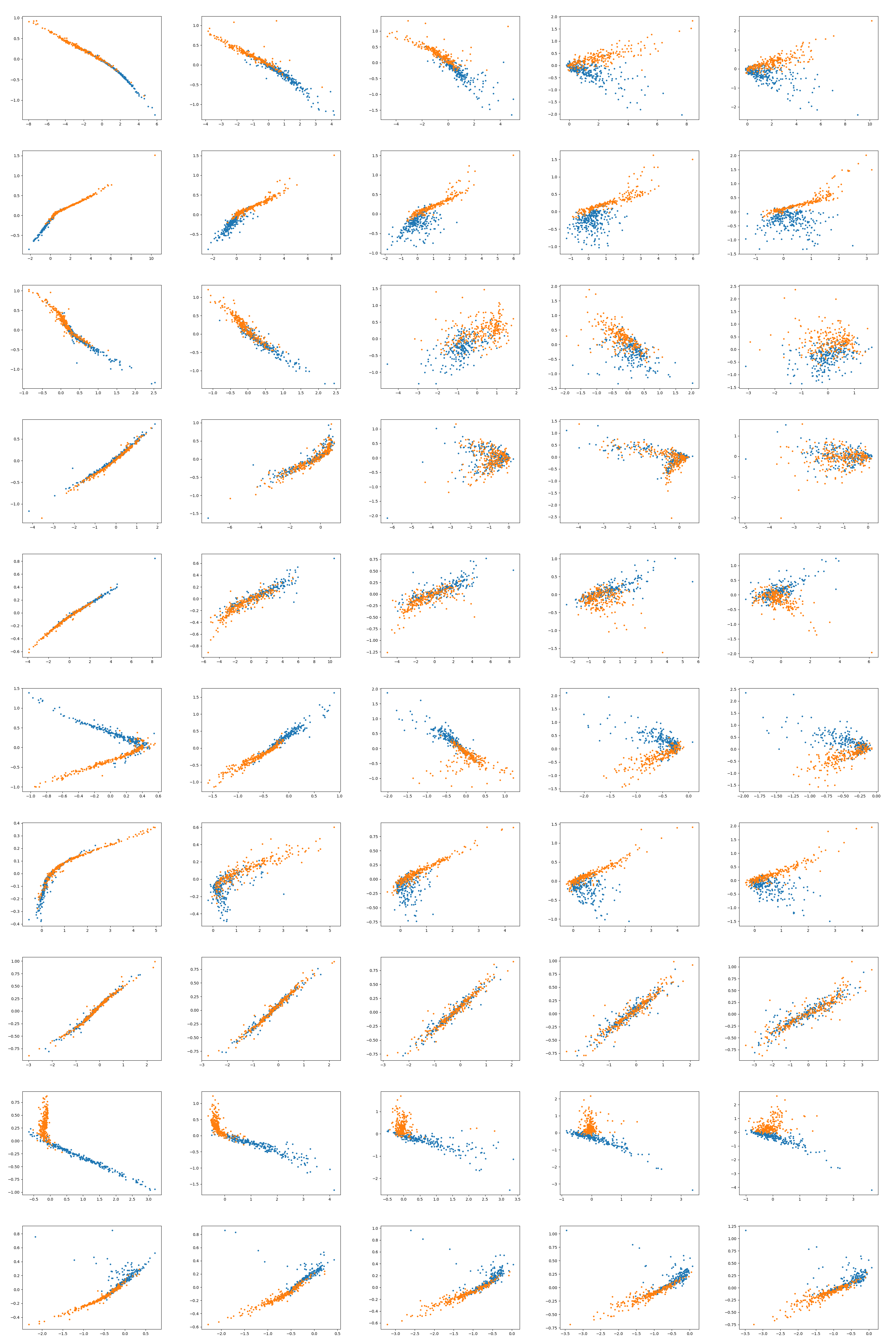}
    \vspace*{-0.1in}
    \caption{Plots of recovered-true latent under \textit{unobserved confounding}. Rows: first 10 nonlinear random models, columns: \textit{proxy} noise level.}
    \vspace*{-0.1in}
\end{figure}

\newpage

\begin{figure}[h] 
    \vspace*{-0.1in}
    \centering
    \includegraphics[width=.9\textwidth]{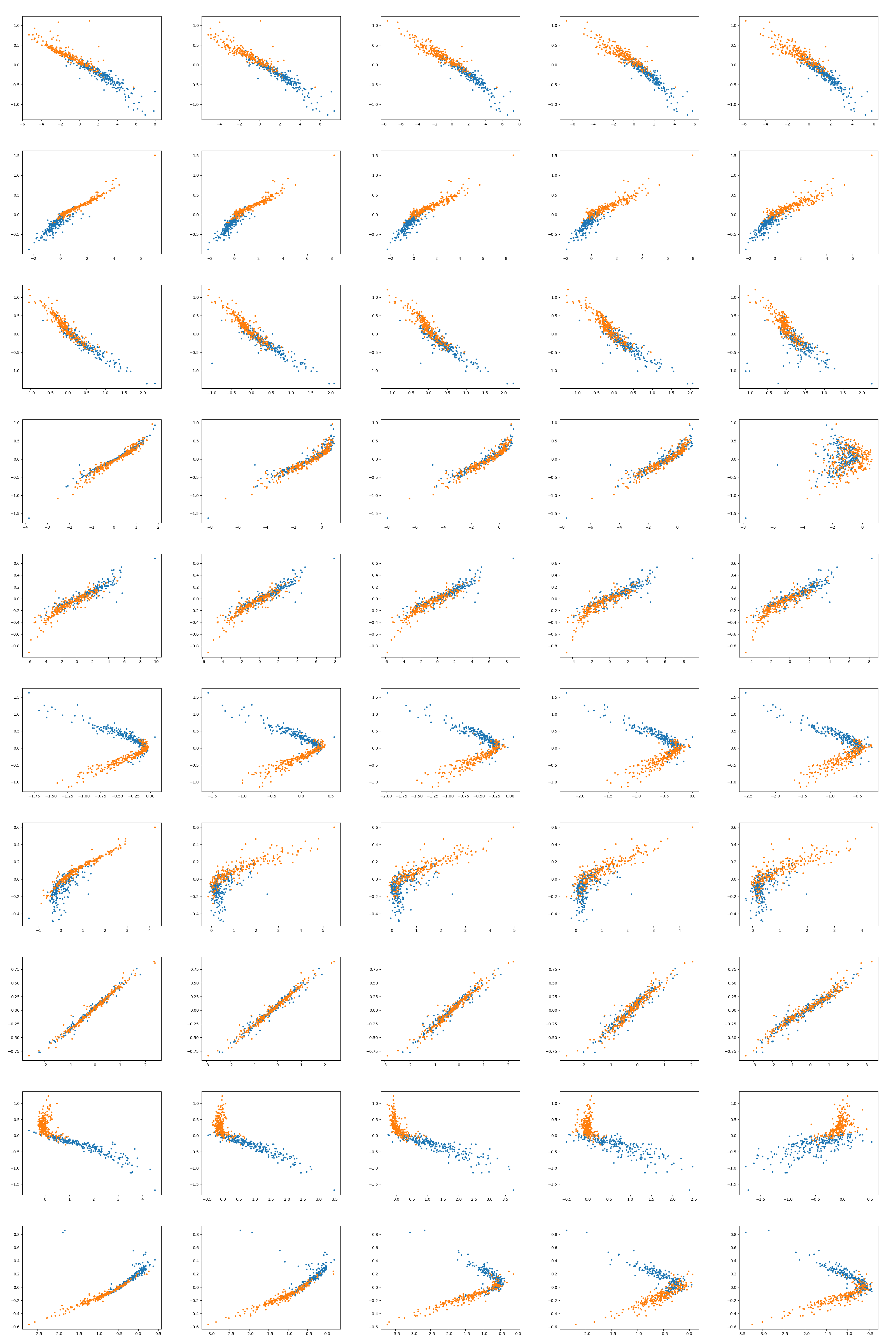}
    \vspace*{-0.1in}
    \caption{Plots of recovered-true latent under \textit{unobserved confounding}. Rows: first 10 nonlinear random models, columns: \textit{outcome} noise level.}
    \vspace*{-0.1in}
\end{figure}

\begin{figure}[h] 
    \vspace*{-0.1in}
    \centering
    \includegraphics[width=.9\textwidth]{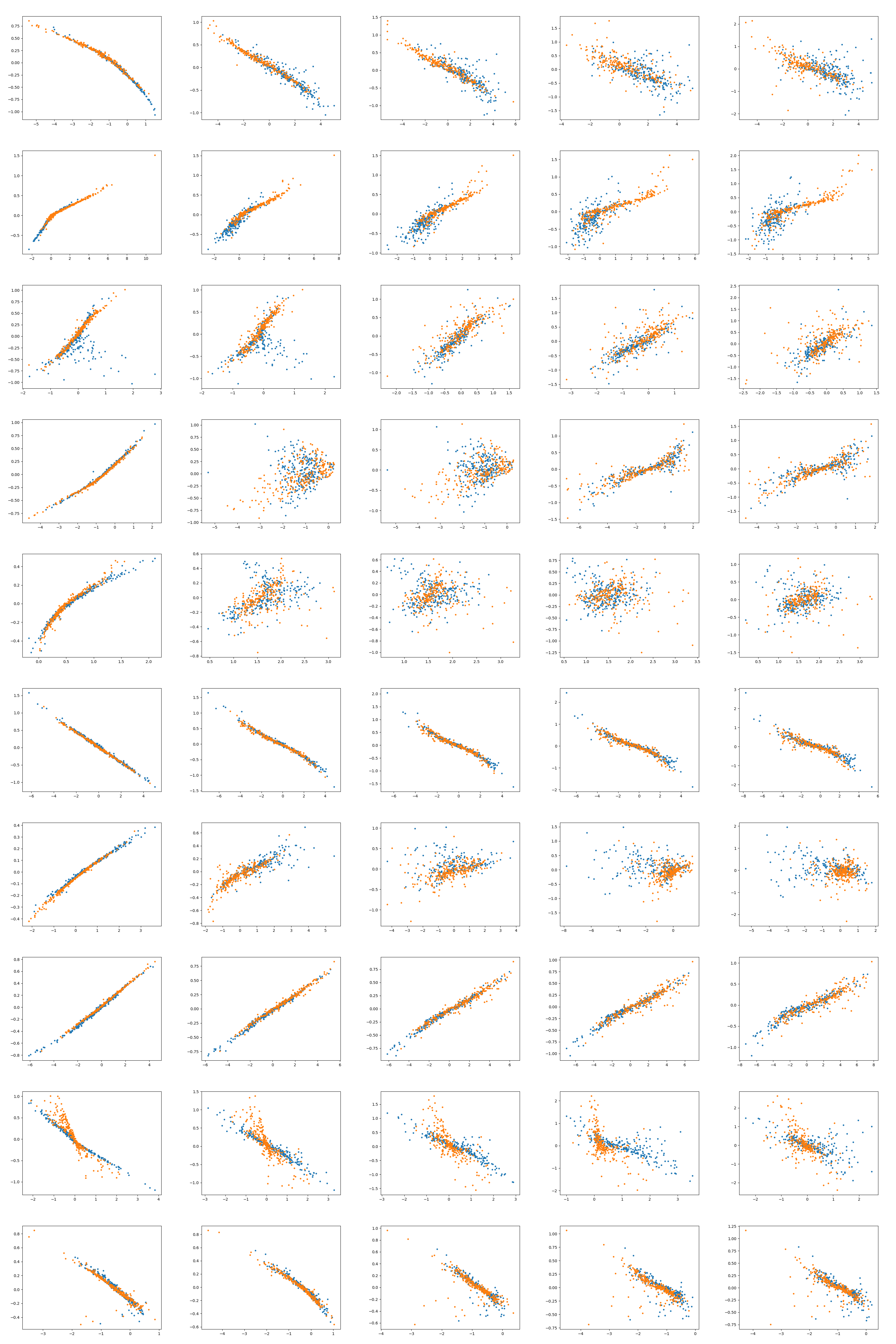}
    \vspace*{-0.1in}
    \caption{Plots of recovered-true latent when \textit{ignorability} holds. Rows: first 10 nonlinear random models, columns: \textit{proxy} noise level.}
    \vspace*{-0.1in}
\end{figure}

\begin{figure}[h] 
    \vspace*{-0.1in}
    \centering
    \includegraphics[width=.9\textwidth]{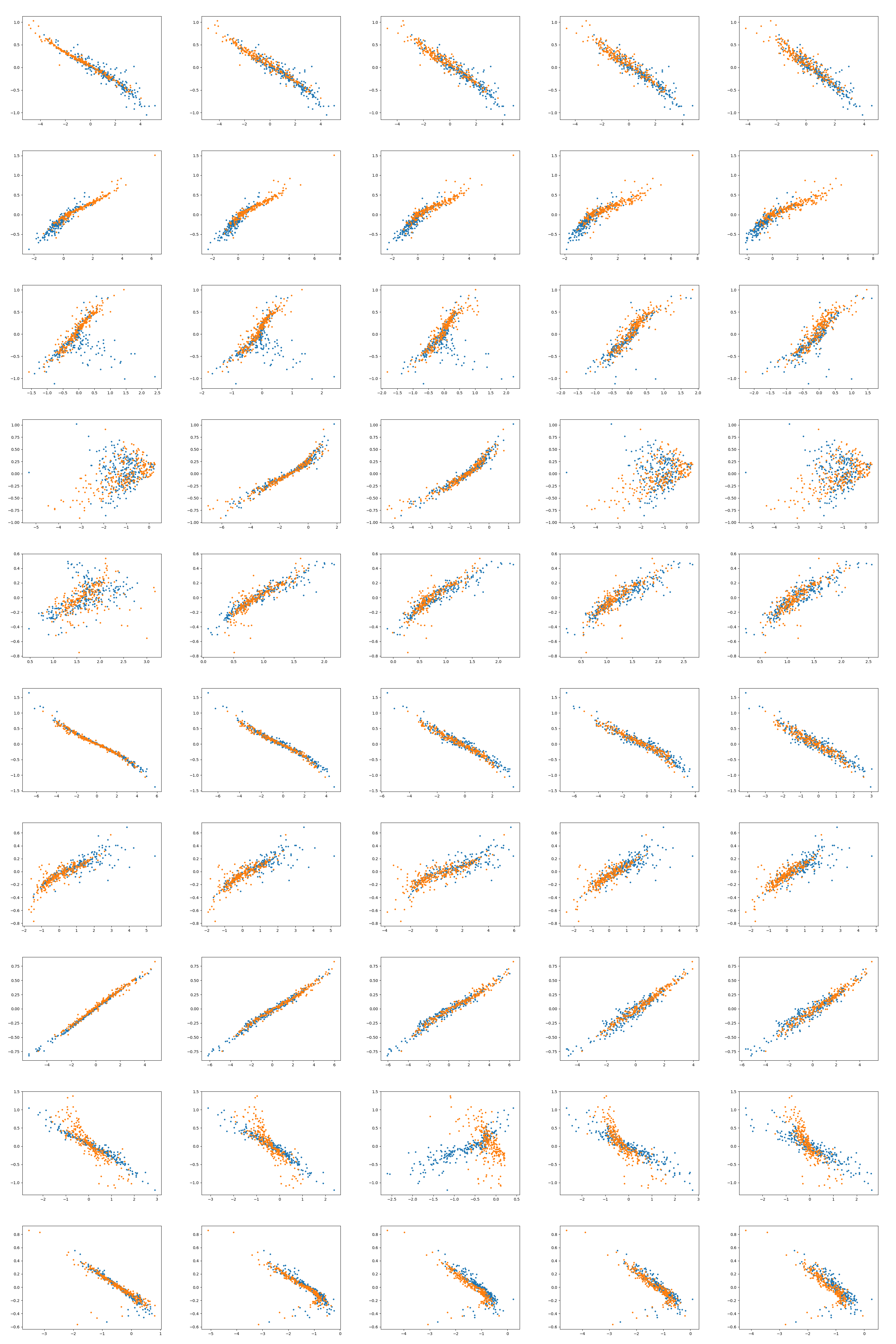}
    \vspace*{-0.1in}
    \caption{Plots of recovered-true latent when \textit{ignorability} holds. Rows: first 10 nonlinear random models, columns: \textit{outcome} noise level.}
    \vspace*{-0.1in}
\end{figure}

\begin{figure}[h] 
    \vspace*{-0.1in}
    \centering
    \includegraphics[width=.9\textwidth]{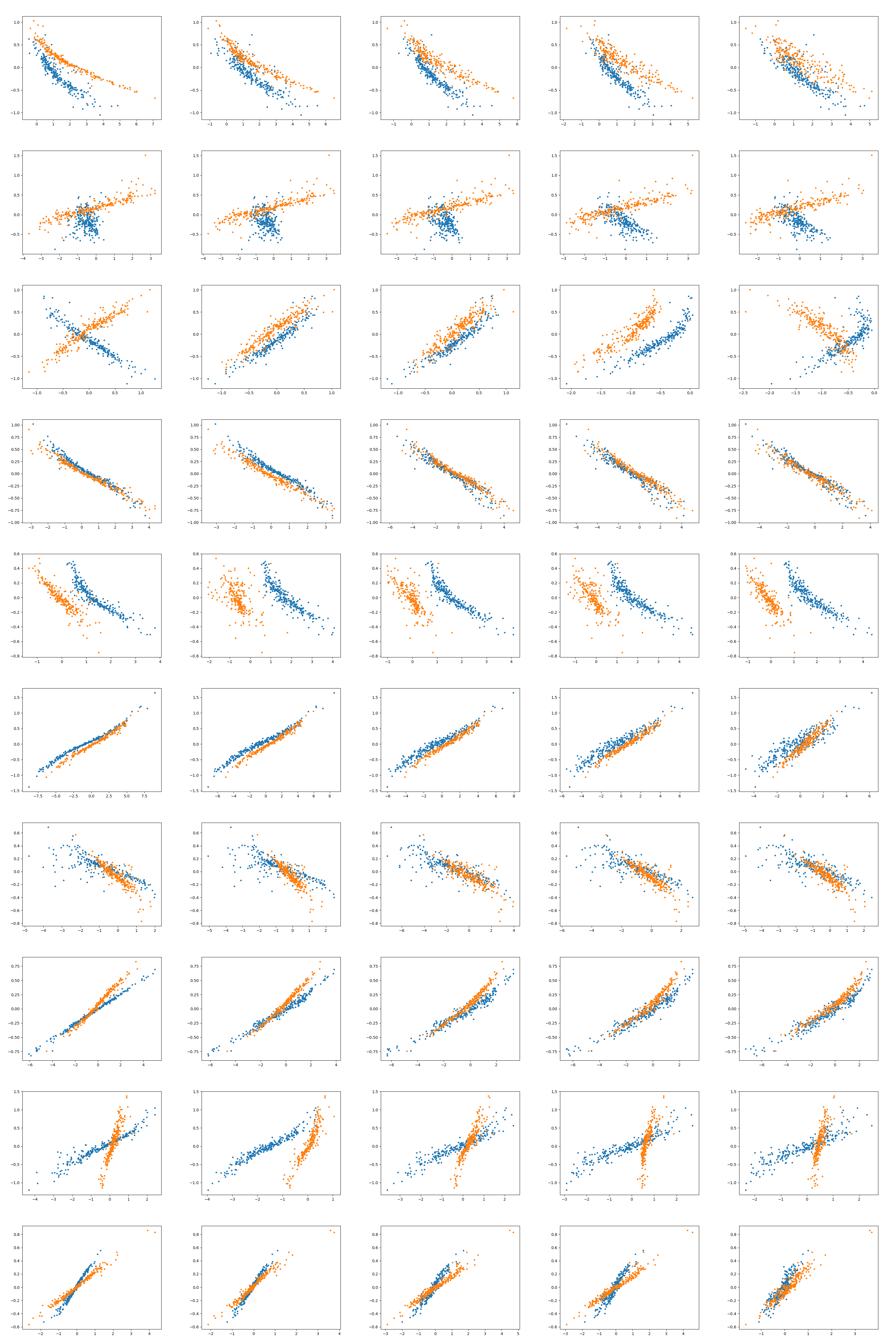}
    \vspace*{-0.1in}
    \caption{Plots of recovered-true latent when \textit{ignorability} holds. Conditional prior \textit{depends} on $t$. Rows: first 10 nonlinear random models, columns: \textit{outcome} noise level. Compare to the previous figure, we can see the transformations for $t=0,1$ are \textit{not} the same, confirming our Theorem \ref{th:bestimation}.}
    \vspace*{-0.1in}
\end{figure}

\begin{figure}[h] 
    \vspace*{-0.1in}
    \centering
    \includegraphics[width=.9\textwidth]{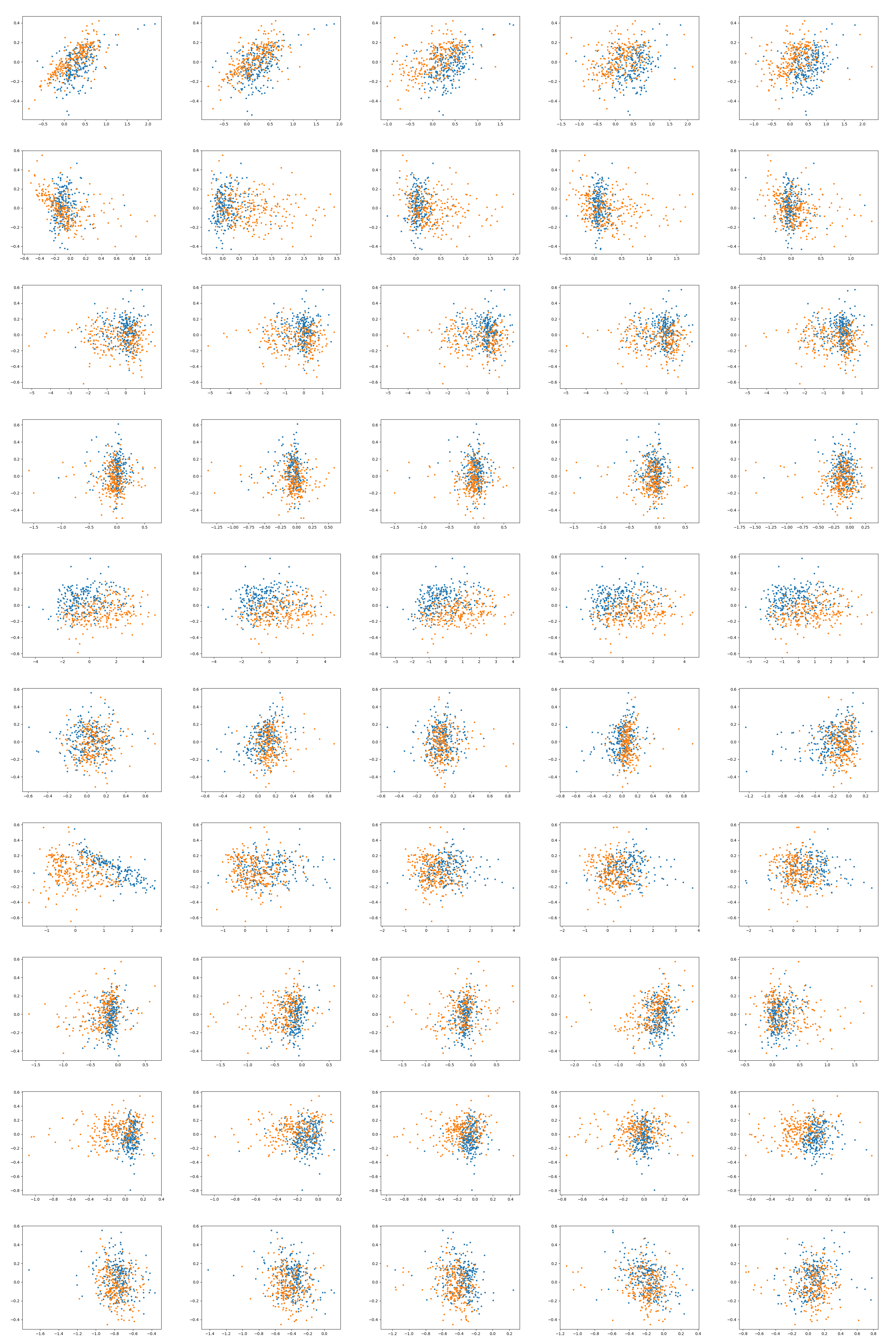}
    \vspace*{-0.1in}
    \caption{Plots of recovered-true latent on \textit{IVs}. Rows: first 10 nonlinear random models, columns: \textit{outcome} noise level.}
    \vspace*{-0.1in}
\end{figure}

\end{document}